\documentclass{article}

\usepackage{PRIMEarxiv}

\usepackage[utf8]{inputenc} 
\usepackage[T1]{fontenc}    
\usepackage{url}            
\usepackage{booktabs}       
\usepackage{nicefrac}       
\usepackage{microtype}      
\usepackage{lipsum}
\usepackage{fancyhdr}       
\usepackage{graphicx}       
\graphicspath{{media/}}     
\usepackage{multirow}%
\usepackage{amsmath,amssymb,amsfonts}%
\usepackage{amsthm}%
\usepackage{mathrsfs}%
\usepackage[title]{appendix}%
\usepackage{textcomp}%
\usepackage{manyfoot}%
\usepackage{algorithm}%
\usepackage{algorithmicx}%
\usepackage{algpseudocode}%
\usepackage{listings}%
\usepackage[table]{xcolor}
\usepackage{lineno}

\theoremstyle{plain}
\newtheorem{proposition}{Proposition}

\newtheorem{corollary}{Corollary}
\theoremstyle{remark}
\newtheorem{remark}{Remark}

\usepackage{hyperref}

\title{Dynamics-Aware Weighting for Deep Learning Forecasts of Chaotic Systems
}

\author{
  Zhou Fang \\
  Department of Mechanical Engineering \\
  National University of Singapore \\
  Singapore\\
  \texttt{zhou.fang@u.nus.edu} \\
   \And
  Gianmarco Mengaldo \\
  Department of Mechanical Engineering \\ 
  Department of Mathematics \\
  National University of Singapore \\
  Singapore\\
  \texttt{mpegim@nus.edu.sg} \\
}

\begin{document}
\maketitle

\begin{abstract}
Deep learning surrogates have become powerful tools for simulating and forecasting complex dynamical systems, yet their utility remains limited by catastrophic error accumulation during long-term autoregressive rollouts.
This behavior is partly tied to the nature of the underlying systems: chaotic spatiotemporal systems visit phase space unevenly, with dynamics dominated by recurrent, low-dimensional quiescent states and characterized by rare and dynamically complex regime transitions.
Trained under a sample-wise uniform objective, standard neural surrogates allocate their finite capacity to the statistically more numerous low-dimensional quiescent states, systematically under-representing the transient regimes that trigger disproportionate, localized errors.
Existing imbalanced-regression methods tackle this issue by reweighting samples according to target-space density.
However, statistical target-space rarity does not coincide with the intrinsic dynamical rarity encoded in the recurrence geometry of the attractor.
To address this, we introduce Dynamics-Aware Weighting (DAW), a data-centric objective reweighting framework.
Using the local dimension $d$ from dynamical systems theory as an \textit{a priori} measure of a state's dynamical complexity, DAW reshapes the loss landscape to allocate representational capacity toward the sparse, high-$d$ regimes where forecast errors are systematically large.
On the chaotic KS equation, DAW consistently outperforms uniform training as well as weighting based on target-space rarity, and its randomly permuted ablation, reducing long-term autoregressive error relative to all baselines.
Event-level analysis shows that DAW achieves this by suppressing the localized error amplifications incurred during sharp jumps in the local dimension $d$, which typically accompany complex physical processes such as wave-merging in the KS system.
\end{abstract}


\section{Introduction}\label{sec:intro}

Deep learning (DL) surrogates have become powerful tools for modeling and forecasting complex dynamical systems, with applications spanning weather and climate, fluid dynamics, and finance, among others~\cite{pathak2018model, reichstein2019deep, chattopadhyay2020data, rasp2020weatherbench, lario2022neural, wang2026condensnet, bracco2025machine}. 
In many settings, these models match or exceed the accuracy of traditional statistical and physics-based approaches at a lower computational cost~\cite{willard2022integrating}. 
However, they remain prone to error accumulation over long-term autoregressive rollouts~\cite{huang2026physicscorrect,  kohl2026benchmarking}. 

This behavior is partly tied to the nature of the underlying systems. 
The phase space of a chaotic attractor is visited unevenly: some regions recur frequently, while others are reached only rarely~\cite{faranda2017dynamical, fang2025dynamical}. 
The Lorenz-63 system is a common example -- its trajectory spends most of its time spiraling within two lobes, whereas the switches between them occur far less often~\cite{alberti2023scale}. 
Trained with a sample-wise uniform objective such as the mean squared error (MSE), a DL surrogate implicitly treats all states as equally informative and allocates its approximation capacity in proportion to the sampling measure, fitting the frequently recurring states while under-representing the rarer ones.

This difficulty is an instance of \emph{imbalanced regression}, in which the data follow a non-uniform distribution and a model that minimizes average error favors the densely sampled majority, performing poorly on the sparse regions of the target space~\cite{yang2021delving, steininger2021density}. 
The mismatch has been observed in both simulated and real-world dynamical systems, yet has so far received limited attention~\cite{puetz2026deconstructing}. 

A common response to this imbalance is to reweight the training objective so that under-represented states contribute more, assigning each sample a scalar importance signal. 
Two lines of work follow this strategy, distinguished by whether the importance signal is drawn from the model or from the data. 
The first line of work relies on \emph{model-centric} signals. 
Motivated by the observation that not all samples contribute equally to learning, and that not all are learned equally well~\cite{katharopoulos2018not}, these importance-sampling schemes prioritize challenging or high-impact samples through \textit{posterior} quantities such as gradients or losses~\cite{katharopoulos2018not, alain2015variance, lin2024not}. 
Originally developed in computer vision to handle class imbalance, such schemes have been used to accelerate training and have gained traction in large language models~\cite{lin2024not, mindermann2022prioritized}. 
However, since the signal is derived from the surrogate itself, it is sensitive to the reactive feedback of the model during training. 
The second line of work relies on \emph{data-centric} signals. 
Here, deep imbalanced regression compensates for rare target values in two ways. One is by estimating the target density with kernel smoothing and up-weighting low-density regions, as in label distribution smoothing~\cite{yang2021delving}, DenseWeight~\cite{steininger2021density}, and SMOGN~\cite{branco2017smogn}. 
The other is by rescaling the loss to balance contributions across the target range~\cite{ren2022balanced}.
Computed once from the data, such signals avoid the model-feedback loop issue of model-centric signals.
Most of these methods reduce each sample to a single aggregate statistic of the target value, whose choice is largely a matter of convention. 
More specifically, depending on the system, the $L_1$ norm, $L_2$ norm, or entropy could each serve, with no uniform criterion or physical insight for selecting among them. 
Moreover, such hand-picked statistics have been shown to transfer poorly to high-dimensional systems~\cite{ren2022balanced}. 

Together, these observations point to a gap, which is the absence of dynamically meaningful indicators that describe the intrinsic imbalance of the underlying system, and of a principled, data-centric scheme designed to exploit them. 

Dynamical systems theory offers exactly such indicators. 
Indeed, the imbalance can be well characterized in phase space by indicators rooted in extreme value theory and Poincar\'e recurrences -- the local dimension $d$ and the inverse persistence $\theta$. 
Previous work empirically shows that forecast errors are not uniform across the dynamical regimes characterized by $d$ and $\theta$, but instead concentrate disproportionately in the higher-$d$ and higher-$\theta$ states~\cite{fang2025dynamical}; of the two, the association with $d$ is stronger and more robust. 
The local dimension $d$ measures how the invariant measure scales within a small neighborhood of a given state or, equivalently, how the frequency with which the trajectory recurs near that state scales with the size of the neighborhood~\cite{faranda2017dynamical, faranda2024statistical}.
This recurrence-based definition gives $d$ a clear physical meaning tied to the geometry of the attractor: low $d$ marks geometrically dense, frequently revisited regions with many historical analogues, whereas high $d$ marks sparse, rarely visited regions with few such analogues, where the local dynamics are more complex and harder to resolve.
The local dimension thus provides a model-free, \textit{a priori} measure of dynamical complexity, computed once from the historical trajectory and requiring no model feedback.
The underlying framework of Poincar\'e recurrences has since been extended to spatiotemporal data~\cite{dong2025spatio} and leveraged to develop predictability indicators~\cite{dong2025time, yang2026hierarchy} and to understand weather pattern changes~\cite{dong2024indo}.
Moreover, the same construction applies uniformly across systems, from low-dimensional attractors such as Lorenz-63 to high-dimensional geophysical fields~\cite{faranda2024statistical}.

Building on this idea, we introduce Dynamics-Aware Weighting (DAW), a data-centric reweighting framework. 
DAW leverages the local dimension $d$ to reshape the loss landscape. 
By assigning higher weight to rare, high-$d$ regimes, DAW encourages the network to devote more capacity to dynamically complex states, where forecast errors tend to be larger. 
We evaluate DAW on the chaotic Kuramoto-Sivashinsky (KS) equation against three baselines that share the same backbone and training budget: uniform training, statistical density weighting (DenseWeight~\cite{steininger2021density}), and a shuffled-weight ablation (RandomWeight). 
Over autoregressive rollouts, DAW attains the lowest error and the slowest loss of spatial correlation, delaying spatial correlation collapse the longest while leaving the intrinsic limits of chaotic predictability intact. 
Moreover, the event-level analysis suggests that DAW achieves this by effectively suppressing the localized error jumps incurred at sharp increases in $d$, which, for KS, coincide with physical processes such as wave-merging. 
This advantage is not specific to the KS geometry. 
On the low-dimensional Lorenz-63 system, the same ordering of methods holds, and DAW delays correlation collapse by roughly 80\% relative to uniform training (Appendix~\ref{si:daw_lorenz}).
A practical strength of DAW is its plug-and-play nature. 
Since $d$ is computed offline, it introduces no architectural change and no inference overhead, leaving the surrogate as lightweight as its unweighted counterpart and ready to be combined with a range of neural architectures.

\section{Methodology}\label{sec:methods}
In the following, we introduce the local dimension (Section~\ref{subsec:d-theory}), which is then used to construct DAW (Section~\ref{subsec:daw}), together with its theoretical motivation (Section~\ref{subsec:daw-theory}).

\subsection{Phase space recurrence and local dimensions $d$}\label{subsec:d-theory}
To quantitatively isolate dynamically critical regimes from trivial backgrounds, we characterize the system's dynamics in phase space through the local dimension indicator $d$.
Formally, the local dimension $d_{\mu}(x)$ of an invariant measure $\mu$ at a state $x$ in the phase space $M$ is defined as
\begin{equation}
    d_{\mu}(x) = \lim_{r\rightarrow 0}\frac{\log\mu(B(x,r))}{\log r},
    \label{eq:local_dim_limit}
\end{equation}
where $B(x,r)$ denotes a ball centered at $x$ with radius $r$ (see also Ref.~\cite{faranda2024statistical}).
This quantity describes how the probability mass of the invariant measure scales within a small neighborhood around the state $x$.

In practice, observational data and numerical simulations are inherently finite.
Therefore, $d$ is estimated by combining the Poincar\'{e} recurrence theorem with Extreme Value Theory (EVT).
For a given reference state $z$ (e.g., a specific spatiotemporal snapshot) and a system trajectory $x(t)$, the negative logarithmic distance is defined as:
\begin{equation}
    g(x(t)) = -\log\big(\text{dist}(x(t), z)\big).
    \label{eq:log_return}
\end{equation}
Requiring that a point on the orbit falls within a small ball of radius $e^{-u}$ around the reference state $z$ is equivalent to requiring that $g(x(t))$ exceeds a high threshold $u(q)$, corresponding to a high quantile $q$, which is set to 0.99 in this work.
The probability of these threshold exceedances converges to the Generalized Pareto Distribution (GPD) family:
\begin{equation}
\begin{aligned}
X &= g(x(t)), \\
\mu\big((X-u(q)) >y \mid X&\geq u(q)\big)
\approx \exp\left(-\frac{y}{\sigma}\right),
\end{aligned}
\label{eq:gpd}
\end{equation}
where the scale parameter $\sigma$ depends explicitly on the reference state $z$ (see also Refs.~\cite{faranda2024statistical,faranda2017dynamical}).
The local dimension $d$, inversely related to the scale parameter of Eq.~\eqref{eq:gpd}, is finally defined as:
\begin{equation}
    d(z) = \frac{1}{\sigma}.
    \label{eq:final_d}
\end{equation}
We emphasize that Eq.~\eqref{eq:final_d} is a finite-resolution estimate at the working quantile $q$ rather than the limit in Eq.~\eqref{eq:local_dim_limit}: for ergodic measures with exact dimension that limit is constant almost everywhere, so the state dependence exploited here reflects the local geometry of the attractor at the scales resolved by the data.
By definition, $d$ measures the geometric sparsity of the attractor around the reference state $z$.
Smaller values of $d$ correspond to regions of the attractor that are geometrically dense and frequently revisited; here few degrees of freedom are active, and the local dynamics remain comparatively simple. Larger values of $d$ instead mark sparse regions with fewer historical analogues, where more degrees of freedom are simultaneously active and the local dynamics are correspondingly more complex and harder to resolve. This asymmetry carries over directly to neural surrogates. Low-$d$ states form highly recurrent regimes whose bulk dynamics dominate the trajectory by frequency of occurrence; owing to this abundance, networks learn to represent them with ease, and may even overfit to their simplicity and redundancy. High-$d$ states, by contrast, occupy the right tail of the dimension distribution: precisely the configurations that demand the most representational capacity are also those most easily neglected under uniform training objectives, leading to systematically higher forecast errors~\cite{fang2025dynamical}.

\subsection{Dynamics-aware weighting (DAW)}\label{subsec:daw}
The construction of Dynamics-Aware Weighting (DAW) proceeds in four steps and leverages the local dimension $d$ just introduced.

\textit{Step 1: Density Estimation of the Local Dimension $d$.}
Let $\mathcal{D} = \{ (x_i, y_i) \}_{i=1}^N$ be the training set. 
For each sample $i$, we define a localized dynamics indicator $d_i \in \mathbb{R}$ by concatenating the input-output pair $(x_i, y_i)$ into a higher-dimensional phase space and computing its local dimension via Eq.~\eqref{eq:final_d}.
The indicator $d_i$ thus quantifies the dynamical complexity of each sample pair.

We then estimate the empirical probability density function (PDF) $P(d)$ over the training set using Kernel Density Estimation (KDE):
\begin{equation}
    P(d) = \frac{1}{N h} \sum_{i=1}^{N} K\left(\frac{d - d_i}{h}\right),
\end{equation}
where $K(\cdot)$ is the Gaussian kernel and $h$ is the bandwidth, set by Silverman's rule of thumb as in Section~\ref{sec:baselines}.
To ensure numerical stability, the density is min-max normalized to bounded values $P'(d) \in [0, 1]$.
The resulting density characterizes how local dimension is distributed across the training set, exposing its inherent imbalance.
\begin{proposition}[Invariance of the local dimension under input--output pairing]
\label{prop:lift}
Let $\mu$ be the invariant measure on the attractor $\mathcal{A} \subset M$ and let $f$ be Lipschitz on $\mathcal{A}$ with constant $L_f \ge 1$. 
Define the lift $\Phi(x) = (x, f(x))$ and let $\nu = \Phi_{\ast}\mu$ be the pushforward of $\mu$ onto the graph of $f$. 
Then, for every $x \in \mathcal{A}$ at which $d_\mu(x)$ exists
\[
d_\nu\big(\Phi(x)\big) = d_\mu(x).
\]
\end{proposition}
\begin{proof}
Equip the product space with the metric $\rho\big((u,u'),(v,v')\big) = \max\{\operatorname{dist}(u,v),\operatorname{dist}(u',v')\}$.
For $u, v \in \mathcal{A}$
\[
\operatorname{dist}(u,v)
\;\le\;
\rho\big(\Phi(u),\Phi(v)\big)
\;\le\;
L_f \operatorname{dist}(u,v),
\]
where the upper bound uses the Lipschitz property of $f$. 
Consequently
\[
B\!\left(x, r/L_f\right)
\;\subseteq\;
\Phi^{-1}\!\Big(B_\rho\big(\Phi(x), r\big)\Big)
\;\subseteq\;
B(x, r),
\]
and, since $\nu(E) = \mu\big(\Phi^{-1}(E)\big)$
\[
\mu\big(B(x, r/L_f)\big)
\;\le\;
\nu\Big(B_\rho\big(\Phi(x), r\big)\Big)
\;\le\;
\mu\big(B(x, r)\big).
\]
Taking logarithms and dividing by $\log r < 0$ reverses the inequalities. 
The right bound converges to $d_\mu(x)$ by assumption. 
For the left bound, set $s = r/L_f$ and write $\log \mu(B(x,s)) / \log r = \log \mu(B(x,s)) / \log s \cdot \log s/\log r$, where $\log s / \log r = 1 - \log L_f / \log r \to 1$ as $r \to 0$. 
Both bounds therefore converge to $d_\mu(x)$, and the limit defining $d_\nu(\Phi(x))$ exists and equals $d_\mu(x)$. 
Since all norms on the finite-dimensional product are equivalent, the same conclusion holds for the Euclidean metric on the concatenated vector used in practice.
\end{proof}
Proposition~\ref{prop:lift} shows that the lift $\Phi$ is bi-Lipschitz onto the graph of $f$ and therefore preserves the pointwise local dimension; in this sense, the pairing is asymptotically redundant. 
At the finite scales probed by the estimator, however, the paired and state-only indicators may differ.

\textit{Step 2: Base Density Weighting on $d$.}
To counteract the under-representation of low-density samples, we first assign each sample a base weight that decreases affinely with its normalized probability density, following the density-based formulation of DenseWeight~\cite{steininger2021density} (the DenseWeight baseline is defined in Section~\ref{sec:baselines}):
\begin{equation}
    w_{i}^{\text{base}} = \max \Big(1 - \alpha P'(d_i), \, \epsilon \Big),
\end{equation}
where $\alpha \ge 0$ is the intensity parameter controlling how aggressively low-density regions are upweighted, and $\epsilon = 10^{-8}$ is a small constant preventing zero weights.
For all main results we use the same parameter $\alpha = 1.0$ for both DAW and the baseline DenseWeight (the default value used in the original DenseWeight work~\cite{steininger2021density}; the optimal $\alpha$ is task dependent).
A sensitivity analysis over $\alpha$ is included in~Appendix~\ref{si:sensitivity_alpha}.

\textit{Step 3: Weight Reshaping with Right-Tail (high-$d$) Focus.}
Because the density-based weight in Step 2 acts on the empirical distribution $P(d)$, it indiscriminately up-weights not only the rare high-$d$ states associated with critical transitions but also the infrequent quiescent low-$d$ states that carry little dynamical complexity.
To align the weighting with the physical meaning of $d$, we tilt it toward the high-$d$ tail by multiplying each base weight with the min-max normalized local dimension $\tilde{d}_i \in [0,1]$:
\begin{subequations}\label{eq:reshaping}
\begin{align}
    \tilde{d}_i &= \frac{d_i - d_{\min}}{d_{\max} - d_{\min}}, \label{eq:dtilde} \\[0.8em]
    w_{i}^{\text{reshaped}} &= \max\big( w_{i}^{\text{base}} \, \tilde{d}_i ,\ \epsilon \big) . \label{eq:wreshaped}
\end{align}
\end{subequations}
Since $\tilde{d}_i$ is smallest for the most quiescent states and approaches one for the rarest high-$d$ states, the tilt in Eq.~\eqref{eq:wreshaped} suppresses the left tail of $P(d)$ while leaving the weights of critical transitions, such as wave-merging events, essentially unchanged; after renormalization, the training emphasis thus shifts toward the high-$d$ tail, although the combined weight need not be monotone in $d$ near the mode of $P(d)$, as visible in Fig.~\ref{fig:daw_weights}.
The floor in Eq.~\eqref{eq:wreshaped}, applied before the global normalization of Step 4 with the same $\epsilon = 10^{-8}$ used throughout, guarantees strictly positive weights, so every sample, including the state attaining $d_{\min}$, retains a small but nonzero contribution to the loss.

\textit{Step 4: Global Weight Normalization.}
Finally, to preserve the global gradient scale and maintain the stability of optimizers, the reshaped weights are explicitly normalized so that their expectation over the training dataset equals $1.0$:
\begin{equation}
    w_i = \frac{w_{i}^{\text{reshaped}}}{\frac{1}{N} \sum_{j=1}^{N} w_{j}^{\text{reshaped}}}.
\end{equation}
The final calibrated objective for the neural surrogates $\hat{f}_\theta$ is the sample-wise weighted MSE:
\begin{equation}\label{eq:daw_loss}
    \mathcal{L}_{\text{DAW}}(\theta) = \frac{1}{N} \sum_{i=1}^{N} w_i \cdot \big\| \hat{f}_\theta(x_i) - y_i \big\|_2^2.
\end{equation}
Fig.~\ref{fig:daw_weights} illustrates the empirical density $P(d)$ with the weight profiles over different $\alpha$.
The distribution $P(d)$ is sharply peaked over the recurrent low-$d$ bulk and trails off into a thin high-$d$ tail, where the increasing weight is assigned.
At $\alpha = 0$, the allocation reduces to the linear $\tilde{d}$-tilt; as $\alpha$ increases, the density term progressively withdraws weight from the over-represented states near the mode of $P(d)$ and redirects it toward the sparse high-$d$ tail.
\begin{figure}[h]
\centering
\includegraphics[width=0.6\linewidth]{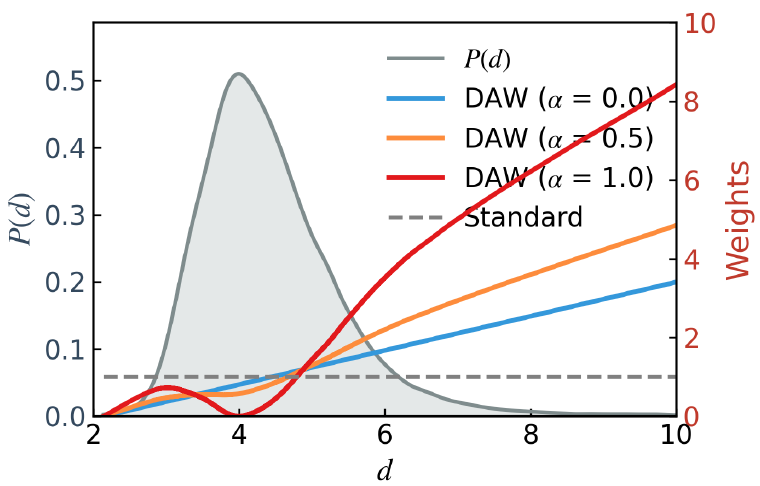}
\caption{\textbf{Visualization of DAW weight distribution.} The left axis displays the empirical probability density $P(d)$ of the local dimension $d$, highlighting the severe imbalance between highly recurrent, low-complexity states and the sparse, right-skewed tail of extreme regime transitions.
The right axis shows the corresponding loss weights allocated by different strategies. 
The Standard baseline applies a naive uniform weight across all regimes, whereas DAW anchors the loss landscape to the underlying dynamical complexity. 
As the intensity parameter $\alpha$ increases from DAW ($\alpha=0.0$) to DAW ($\alpha=1.0$), the density-weighting component is progressively engaged on top of the $\tilde{d}$-tilt, concentrating loss mass onto the high-$d$ tail and increasingly amplifying rare regime transitions relative to the recurrent low-$d$ bulk.}
\label{fig:daw_weights}
\end{figure}

\subsection{Theoretical motivation for DAW: error propagation and imbalanced regression}\label{subsec:daw-theory}

The theoretical construction in this section motivates the design of DAW and delimits its scope; it does not directly derive the specific weighting rule of Section~\ref{subsec:daw}, which remains an empirical construction guided by it.
We start from defining useful quantities. 
More specifically, let $f$ be the true one-step flow map and $\hat{f}$ the surrogate map learned by the neural network, and let $\{x_k\}$ be a reference trajectory of the system, with $x_{k+1}=f(x_k)$.
An autoregressive forecast $\{\hat x_k\}$ is produced by iterating $\hat x_{k+1}=\hat f(\hat x_k)$ from an initial condition $\hat x_0=x_0$. \\[-0.5em]
\begin{proposition}[Finite-horizon error decomposition]
\label{prop:decomp}
Let $U$ be a compact convex subset of the discretized state space $\mathbb{R}^{m}$, containing the reference trajectory $\{x_k\}_{k=0}^{n}$ and the rollout $\{\hat{x}_k\}_{k=0}^{n}$, and let $f \in C^2(U)$ with $\sup_{U}\|Df\| \le \Lambda$ and $\sup_{U}\|D^2 f\| \le K$, where $\|\cdot\|$ denotes the Euclidean norm on vectors and the induced operator norm on $Df$ and $D^2 f$. 
Define forecast error at the $k$-th step as $\delta_k=\hat{x}_k-x_k$, the one-step model error at the visited state $\tilde{\varepsilon}_k = \hat{f}(\hat{x}_k) - f(\hat{x}_k)$ and the time-ordered cocycle $M_{k\to n} = Df(x_{n-1})\,Df(x_{n-2})\cdots Df(x_{k+1})$, with the empty product equal to the identity.
Then, for every $n \ge 1$
\begin{equation}\label{eq:decomp}
\delta_n = \sum_{k=0}^{n-1} M_{k\to n}\big(\tilde{\varepsilon}_k + R_k\big),
\qquad
\|R_k\| \le \tfrac{K}{2}\,\|\delta_k\|^{2},
\end{equation}
and consequently
\begin{equation}\label{eq:decomp_bound}
\|\delta_n\| \le \sum_{k=0}^{n-1} \big\|M_{k\to n}\big\| \Big(\|\tilde{\varepsilon}_k\| + \tfrac{K}{2}\|\delta_k\|^{2}\Big).
\end{equation}
\end{proposition}
\begin{proof}
By definition, $\delta_{k+1} = \hat{f}(\hat{x}_k) - f(x_k) = \tilde{\varepsilon}_k + f(\hat{x}_k) - f(x_k)$.
Since $U$ is convex, the segment joining $x_k$ and $\hat{x}_k$ lies in $U$, and Taylor's theorem with integral remainder gives $f(\hat{x}_k) - f(x_k) = Df(x_k)\,\delta_k + R_k$ with $\|R_k\| \le \tfrac{K}{2}\|\delta_k\|^2$. 
Hence $\delta_{k+1} = \tilde{\varepsilon}_k + R_k + Df(x_k)\,\delta_k$, and the claim follows by induction on $n$ from $\delta_0 = 0$, using $M_{k\to n+1} = Df(x_n)\,M_{k\to n}$ for $k \le n-1$ and $M_{n\to n+1} = I$.
The bound is the triangle inequality.
\end{proof}
\begin{corollary}[Validity horizon of the linear regime]
\label{cor:horizon}
Assume $\Lambda > 1$ (note that $\lVert M_{k\to m}\rVert \leq \Lambda^{m-1-k}$ for all $k < m \leq n$ follows from $\sup_U \lVert Df\rVert \leq \Lambda$ by submultiplicativity), let $\bar{\varepsilon} = \max_{k<n} \lVert \tilde{\varepsilon}_k\rVert$, and set $S_m = \frac{\Lambda^m - 1}{\Lambda - 1}$. 
If $2K\bar{\varepsilon}\,S_n^{2} \le 1$, then $\|\delta_m\| \le 2\bar{\varepsilon}\,S_m$ for all $m \le n$. 
In particular, the linearized decomposition dominates the remainder for horizons $n \lesssim \frac{\log(1/\bar{\varepsilon})}{2\log\Lambda}$.
\end{corollary}
\begin{proof}
Induction on $m$. The claim holds at $m=0$. If it holds for all $j < m$, then by Proposition~\ref{prop:decomp}
\[
\|\delta_m\|
\le \sum_{k=0}^{m-1} \Lambda^{m-1-k}
\Big(\bar{\varepsilon} + \tfrac{K}{2}\cdot 4\bar{\varepsilon}^{2} S_k^{2}\Big)
\le \bar{\varepsilon}\,S_m\big(1 + 2K\bar{\varepsilon}\,S_n^{2}\big)
\le 2\bar{\varepsilon}\,S_m,
\]
using $S_k \le S_n$ and $\sum_{k=0}^{m-1}\Lambda^{m-1-k} = S_m$.
\end{proof}
\begin{remark}
Proposition~\ref{prop:decomp} is stated with the one-step error at the predicted state, $\tilde{\varepsilon}_k$, whereas the discussion below and the empirical premise use the error at the reference state, $\varepsilon_k = \hat{f}(x_k) - f(x_k)$.
If the residual $\hat{f} - f$ is Lipschitz on $U$ with constant $L_g$, then $\|\tilde{\varepsilon}_k - \varepsilon_k\| \le L_g \|\delta_k\|$.
The substitution is therefore controlled only when the surrogate is close to $f$ in the $C^1$ sense, so that $L_g\|\delta_k\|$ remains small relative to $\|\varepsilon_k\|$ over the horizon of Corollary~\ref{cor:horizon}, and we use the reference-state error in the discussion under this condition.
\end{remark}
Equation~\eqref{eq:decomp} separates two logically distinct contributions to the forecast error.
The factor $\|\varepsilon_k\|$ measures \emph{where large one-step errors occur}, and is a property of the trained surrogate.
The cocycle norm $\|M_{k\to n}\|$ measures \emph{how strongly an error injected at time $k$ is amplified afterwards}, and is a property of the dynamics alone, independent of the surrogate.
Any state-dependent loss weight must be understood through which of these two factors it acts upon.

The empirical premise of this work is that the one-step error $\|\varepsilon_k\|$ is systematically larger at states of high local dimension $d$~\cite{fang2025dynamical}; we validate this premise directly on the KS system through the dimension-binned ($d$-binned) one-step error analysis reported in Section~\ref{sec:results}.
These rare, geometrically complex states (regime transitions, wave-merging events) are visited infrequently.
Therefore, under the uniform mean-squared objective, the loss is dominated by the dense, low-$d$ majority and the surrogate underfits the high-$d$ states.
In other words, training the neural network with a uniform mean-squared loss yields an under-resolved map that smooths high-$d$ excursions toward the behavior of the abundant neighboring states.

The role of dynamics-aware weighting is to reduce $\|\varepsilon_k\|$ at exactly the times $k$ where it is largest, namely the high-$d$ states, and thereby to shrink every term of the sum in Eq.~\eqref{eq:decomp} that originates there.
The amplification factors $\|M_{k\to n}\|$ along the reference trajectory are fixed by the true flow and cannot be altered by any choice of training loss; DAW therefore acts only on the injected one-step errors. 

We note that Eq.~\eqref{eq:decomp} also admits a distributional reading.
Averaged over many rollouts, the expected forecast error weights each state by how often it is visited (the invariant measure $\mu$) and by how much its one-step error contributes downstream.
Training under the uniform objective minimizes $\mathbb{E}_{\mu}\|\varepsilon\|^2$ and therefore spends model capacity in proportion to $\mu$, that is, on the low-$d$ bulk.
The distribution that actually governs rollout accuracy is tilted toward states whose one-step errors are large and systematic, which the premise above identifies with the rare high-$d$ states.
Dynamics-aware weighting reshapes the effective training measure toward these states, partially correcting the covariate shift between the sampling measure and the rollout-relevant measure~\cite{shimodaira2000improving}.
The local dimension $d$ thus serves as an \emph{a priori}, model-free, and inference-free proxy for where the one-step error will be large, computed once from the geometry of the dynamical system.

\section{Experimental setup}\label{sec:setup}

\subsection{Data}\label{sec:data}
We evaluate the proposed method on the one-dimensional Kuramoto-Sivashinsky (KS) equation, a canonical model of spatiotemporal chaos that exhibits diverse dynamics, including quiescent (near-laminar) regimes, traveling waves, and transient wave-merging events. 

The KS equation is 
\begin{equation}\label{eqn:ks}
    \frac{\partial u}{\partial t} + u \frac{\partial u}{\partial x} + \frac{\partial^2 u}{\partial x^2} + \frac{\partial^4 u}{\partial x^4} = 0,
\end{equation}
where $u$ denotes the observable, $t$ is time, and $x$ is the spatial coordinate defined over the periodic domain $[0, L)$.

The complexity and effective dimensionality of the KS dynamics are controlled by the domain size $L$. 
In this work, we set $L = 22$, following the configuration of~\cite{pathak2018model}.
For $L=22$, the KS equation sustains spatiotemporal chaos and the flow is organized by a small set of invariant solutions: relative equilibria (traveling waves) and relative periodic orbits, connected by heteroclinic orbits between equilibria~\cite{cvitanovic2010state}. Quiescent intervals correspond to passages near these invariant solutions, while transient wave-merging events mark fast excursions between them.
We discretize the spatial domain into $N_x = 64$ uniformly spaced grid points. 
We generate in total $2{,}500{,}000$ samples using a small integration step $dt = 0.01$ to ensure numerical stability, and then discard the first $10{,}000$ steps to remove transients and allow the trajectory to settle onto the attractor.
Next, the generated data is downsampled to $dt = 0.25$ for the forecasting task. 
Time is also reported in Lyapunov time units, $\mathrm{LT} = 1/\lambda_{\max}$, with $\lambda_{\max} \approx 0.05$ for $L = 22$, so that one LT corresponds to approximately $80$ forecast steps.
The final length of the total data is $100{,}000$ (around 1200 Lyapunov time), with the length of the training set being around $70{,}000$ (850 Lyapunov time), and the remaining data divided into validation and test sets, each containing around $15{,}000$ samples. 
The long time sequence used enables a robust $d$ estimation on both distribution shape and relative magnitude, as shown in~Appendix~\ref{si:d_convergence}.
For the main experiments, we use an input length of three steps, and forecast the next one step.

Before training, the data is standardized via Z-score normalization using statistics computed on the training split, and the same statistics are applied to the validation and test splits.

\subsection{Model configuration and training}\label{sec:training}
The network, optimizer, learning-rate schedule, batch size, and maximum number of training epochs are identical across all methods.
We use a multilayer perceptron (MLP) surrogate implemented in PyTorch, with 6 hidden layers and 128 neurons in each hidden layer.
We optimized the model using the Adam optimizer with a batch size of 128, an initial learning rate of $5 \times 10^{-4}$, and a weight decay of $1 \times 10^{-9}$.
The total number of training epochs for each run is fixed at 1000, and the checkpoint with the lowest validation loss is used for evaluation.

The objective of all experiments can be unified as a weighted mean-squared error (WMSE),
\begin{equation}\label{eqn:weighted_loss}
    \mathcal{L} = \frac{1}{N} \sum_{i=1}^{N} w_i \, \ell(y_i, \hat{y}_i),
    \qquad
    \ell(y_i, \hat{y}_i) = \big\lVert \hat{y}_i - y_i \big\rVert_2^2,
\end{equation}
where $y_i$ and $\hat{y}_i$ are the ground-truth and predicted states for training sample $i$, and the non-weighted baseline can be recovered exactly by setting $w_i = 1\,\forall i$.
The definition of the weight construction is detailed in Section~\ref{sec:baselines}.
To preserve the global gradient scale, every weighting scheme is normalized so that the expected weight over all training data is unity, $\mathbb{E}_i[w_i] = 1.0$.

At evaluation, the trained model is tested in closed loop. 
Its prediction is fed back as the input for the next step, producing an autoregressive rollout.

\subsection{Baselines}\label{sec:baselines}
To assess the efficacy of DAW, we compare it against three baselines that share the identical backbone, optimizer, and training budget as described in Section~\ref{sec:training} and differ only in the sample-weighting function $w_i$ applied in Eq.~\eqref{eqn:weighted_loss}. 

\noindent \textit{Standard.}
Standard training treats all samples equally, implicitly assuming all regimes are uniformly informative:
\begin{equation}
    w_{i}^{\text{Standard}} = 1.0, \quad \forall i \in \{1, \dots, N\}.
\end{equation}
%

\noindent \textit{DenseWeight.}
DenseWeight is a state-of-the-art deep imbalanced regression (DIR) method that reweights samples by the statistical rarity of their target value~\cite{steininger2021density}. 
It applies kernel density estimation (KDE) to a scalar target feature $z_i$. 
The bandwidth was determined via Silverman's rule of thumb, $h = 1.06 \hat{\sigma} N^{-1/5}$, where $\hat{\sigma}$ is the standard deviation of the sample and $N$ is the sample size.
For a multivariate field such as KS, we take the spatial $L_2$-norm of the target state to obtain its empirical density $P(z_i)$, and assign weights that decrease affinely with the normalized density:
\begin{equation}
     w_{i}^{\text{DenseWeight}} \propto \max\!\Big(1 - \alpha\, P'(z_i), \, \epsilon \Big),
\end{equation}
where $P'(z_i) \in [0,1]$ is the min-max-normalized $P(z_i)$, $\alpha$ is a weighting intensity factor, and we use $\epsilon = 10^{-8}$ as a floor that prevents zero weights. 
This baseline directly tests whether emphasizing statistical rarity in target space improves model performance.

\noindent \textit{RandomWeight.}
RandomWeight is designed as a shuffled-weight ablation of DAW, isolating whether DAW's effect comes from the \emph{precise alignment} of high weights with dynamically complex states, or merely from the \emph{increased variance} of the gradient landscape that any non-uniform weighting may induce. 
To test this, we take the exact set of weights produced by DAW, $\mathcal{W}^{\text{DAW}} = \{w_1^{\text{DAW}}, \dots, w_N^{\text{DAW}}\}$, and reassign them to samples through a uniformly random permutation $\pi$
\begin{equation}
    w_{i}^{\text{RandomWeight}} = w_{\pi(i)}^{\text{DAW}}, \quad \pi \sim \text{Permutation}(N).
\end{equation}
By construction, RandomWeight shares the identical marginal weight distribution as DAW, but the correspondence between high weights and high-$d$ regimes is destroyed.
Any advantage of DAW over RandomWeight therefore cannot be attributed to the weight magnitudes alone and must instead reflect the alignment of weight with the local geometry of the attractor, noting that second-order quantities depending on the weight-loss covariance, such as the variance of the weighted gradients, are not exactly matched by a permutation.

\subsection{Evaluation metrics}\label{sec:evaluation_metrics}
We quantify predictive accuracy and spatial-correlation consistency with two complementary primary metrics, the Mean Absolute Error (MAE) and the spatial Pearson Correlation Coefficient, and report two auxiliary event-level diagnostics that probe robustness during the system's transitions. Throughout, $y$ denotes the ground-truth state and $\hat{y}$ the model prediction over the autoregressive rollout.

\noindent \textit{Mean Absolute Error (MAE).}
MAE measures the average point-wise magnitude of the prediction error over the full spatiotemporal evaluation domain, providing a direct measure of absolute deviation from the true dynamics:
\begin{equation}
    \text{MAE} = \frac{1}{N_{\mathrm{e}}} \sum_{i=1}^{N_{\mathrm{e}}} \big| \hat{y}_i - y_i \big|,
\end{equation}
where $i$ indexes the $N_\mathrm{e}$ evaluated state entries (all spatial grid points across all rollout steps). 
A lower MAE indicates higher point-wise accuracy.

\noindent \textit{Spatial Pearson Correlation Coefficient (Correlation).}
Because point-wise distance alone does not capture whether a prediction preserves the shape of the field, we additionally report the spatial Pearson correlation, which measures how well the surrogate reproduces the spatial pattern, traveling-wave patterns, and phase alignment of the dynamics:
\begin{equation}
    \text{Correlation} = \frac{\sum_{i=1}^{N_{\mathrm{e}}} (y_i - \bar{y})(\hat{y}_i - \bar{\hat{y}})}{\sqrt{\sum_{i=1}^{N_{\mathrm{e}}} (y_i - \bar{y})^2}\,\sqrt{\sum_{i=1}^{N_{\mathrm{e}}} (\hat{y}_i - \bar{\hat{y}})^2}},
\end{equation}
where $\bar{y}$ and $\bar{\hat{y}}$ are the means of the ground truth and predictions over the evaluated domain. 
The correlation lies in $[-1, 1]$, with values near $1.0$ indicating strong spatial pattern consistency.

\noindent \textit{Cumulative Error (CE) and Overall Win Rate.}
To evaluate robustness specifically during the most dynamically complex regimes (the upper quartile of d) that may drive catastrophic error growth, we introduce two event-level diagnostics. 
We first identify extreme events along the autoregressive rollout as contiguous intervals in which the local dimension $d$ exceeds a high-quantile threshold (the $75$th percentile). 
For each such interval we locate the time index of peak local dimension, $t_{\text{peak}}$, which marks the core of the topological transition (e.g.\ a wave merge), and extract a window of $n_w = 3$ steps on either side of the peak. 
The Cumulative Error (CE) of an event is the sum of absolute point-wise errors over this window:
\begin{equation}
    \text{CE} = \sum_{t=t_{\text{peak}}-n_w}^{t_{\text{peak}}+n_w} \sum_{i=1}^{N_{x}} \big| \hat{y}_{i,t} - y_{i,t} \big|,
\end{equation}
where $y_{i,t}$ and $\hat{y}_{i,t}$ are the ground-truth and predicted states at spatial grid point $i$ and time step $t$. 
Computing CE for every method (DAW, DenseWeight, RandomWeight, and Standard) across all extracted high-$d$ windows, we define the \emph{Overall Win Rate} of a method as the percentage of event windows in which it attains the lowest CE. 
Together, CE and the Overall Win Rate quantify each method's comparative ability to suppress error accumulation during the system's most demanding regimes.

\section{Results}\label{sec:results}

We first validate the empirical premise stated in Section~\ref{subsec:daw-theory}, that the one-step error $\|\varepsilon_k\|$ is systematically larger at states of high local dimension $d$.
Fig.~\ref{fig:d-binned-error} partitions the test set into five equal-mass bins of increasing $d$, each containing $20\%$ of the samples. 
For all four objectives, the mean one-step MAE rises monotonically across the bins, confirming that the states carrying the largest one-step error are the rare, geometrically complex ones. 
Within every bin, both density-aware objectives (DAW and DenseWeight) reduce error across the full range of $d$ relative to uniform training (Standard) and the shuffled-weight ablation
(RandomWeight). 
Two features stand out. 
First, DAW's reduction is not confined to the high-$d$ tail it explicitly targets but also lowers error in the low-$d$ bulk, so the reweighting sharpens the surrogate globally rather than trading bulk accuracy for tail accuracy. 
Second, RandomWeight, which inherits DAW's exact weight magnitudes but not their alignment with $d$, remains closest to Standard among the weighted objectives, indicating that the alignment between weight and topology, rather than weight magnitude alone, drives most of the one-step improvement.
\begin{figure}[htb]
\centering
\includegraphics[width=0.6\linewidth]{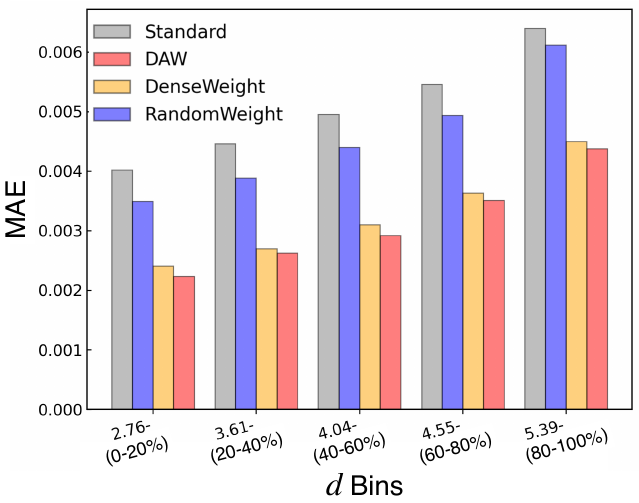}
\caption{\textbf{One-step forecast error binned by local dimension $d$.}
The test set is partitioned into five equal-mass bins of increasing $d$ (each containing $20\%$ of states); bars show the mean one-step MAE of each method within each bin. 
For every training objective the error increases monotonically from the low-$d$ bulk to the high-$d$ tail, validating the premise of Section~\ref{subsec:daw-theory} that one-step error concentrates at high-$d$ states.
DAW attains the lowest error in every bin and reduces error across the full range of $d$ -- including the low-$d$ bulk -- relative to Standard and RandomWeight. }
\label{fig:d-binned-error}
\end{figure}
%

\begin{table}[htb]
    \centering
    \caption{\textbf{Model evaluation results} (mean $\pm$ std over three independent runs). 
    The best-performing method for each metric at every forecast horizon is highlighted in \textbf{bold}.}
    \label{tab:overall_evaluation_results}
    \begin{tabular}{clcc}
        \toprule
        \textbf{Step} & \textbf{Method}
        & \textbf{MAE} $\downarrow$
        & \textbf{Correlation} $\uparrow$ \\
        \midrule
        & DAW           & \textbf{0.0032 $\pm$ 0.0000}
                        & \textbf{1.0000 $\pm$ 0.0000} \\
        & DenseWeight   & \textbf{0.0032 $\pm$ 0.0001}
                        & \textbf{1.0000 $\pm$ 0.0000} \\
        & RandomWeight  & 0.0047 $\pm$ 0.0001
                        & \textbf{1.0000 $\pm$ 0.0000} \\
        \multirow{-4}{*}{1 (0.01\,LT)}
        & Standard      & 0.0058 $\pm$ 0.0006
                        & \textbf{1.0000 $\pm$ 0.0000} \\
                        \midrule
        & DAW           & \textbf{0.1109 $\pm$ 0.0084}
                        & \textbf{0.9876 $\pm$ 0.0020} \\
         
        & DenseWeight   & 0.1320 $\pm$ 0.0195
                        & 0.9825 $\pm$ 0.0050 \\
         
        & RandomWeight  & 0.1586 $\pm$ 0.0053
                        & 0.9736 $\pm$ 0.0031 \\
         
        \multirow{-4}{*}{20 (0.25\,LT)}
        & Standard      & 0.2988 $\pm$ 0.0534
                        & 0.9193 $\pm$ 0.0247 \\
                        \midrule
        & DAW           & \textbf{0.2993 $\pm$ 0.0233}
                        & \textbf{0.8874 $\pm$ 0.0170} \\
        & DenseWeight   & 0.3766 $\pm$ 0.0692
                        & 0.8330 $\pm$ 0.0542 \\
        & RandomWeight  & 0.4328 $\pm$ 0.0189
                        & 0.7759 $\pm$ 0.0283 \\
        \multirow{-4}{*}{40 (0.5\,LT)}
        & Standard      & 0.7049 $\pm$ 0.0835
                        & 0.5503 $\pm$ 0.0690 \\
                        \midrule
         
        & DAW           & \textbf{0.7779 $\pm$ 0.0306}
                        & \textbf{0.4153 $\pm$ 0.0353} \\
         
        & DenseWeight   & 0.8656 $\pm$ 0.0959
                        & 0.3224 $\pm$ 0.1077 \\
         
        & RandomWeight  & 0.9537 $\pm$ 0.0174
                        & 0.2269 $\pm$ 0.0476 \\
         
        \multirow{-4}{*}{80 (1.0\,LT)}
        & Standard      & 1.0635 $\pm$ 0.0184
                        & 0.1419 $\pm$ 0.0468 \\
        \bottomrule
    \end{tabular}
\end{table}

Table~\ref{tab:overall_evaluation_results} summarizes the aggregate performance at four representative rollout horizons. 
At the single-step horizon (Step~1), DAW and DenseWeight perform on par within run-to-run variability and both improve on Standard and RandomWeight, while all four methods retain near-unity spatial correlation.
This near-parity is expected, as one-step accuracy on the training distribution is governed by the quality of the overall learned map rather than by differential coverage of the rare, high-$d$ regimes, which have yet to shape a single-step prediction.

The separation between methods emerges monotonically as the closed-loop rollout extends. 
By $0.25$~LT (Step~20), DAW already leads on both metrics, whereas Standard has degraded markedly, with an MAE nearly triple that of DAW. 
By $0.5$~LT (Step~40) the gap further widens, with DAW preserving a high spatial correlation while Standard falls to the $0.5$ predictability threshold, marking the onset of spatial correlation collapse. 
At $1.0$~LT (Step~80), DAW attains both the lowest accumulated MAE and the highest preserved correlation among all four methods, with DenseWeight, RandomWeight, and Standard following in that order. 
Nonetheless, the intrinsic limits of chaotic predictability still manifest as one should expect, with forecast correlations decaying substantially by $1.0$~LT.
The benefit of DAW is a systematic delay of spatial correlation collapse rather than its elimination. 
Exact values at each horizon are reported in Table~\ref{tab:overall_evaluation_results}. 
\begin{figure}[H]
\centering
\includegraphics[width=0.8\linewidth]{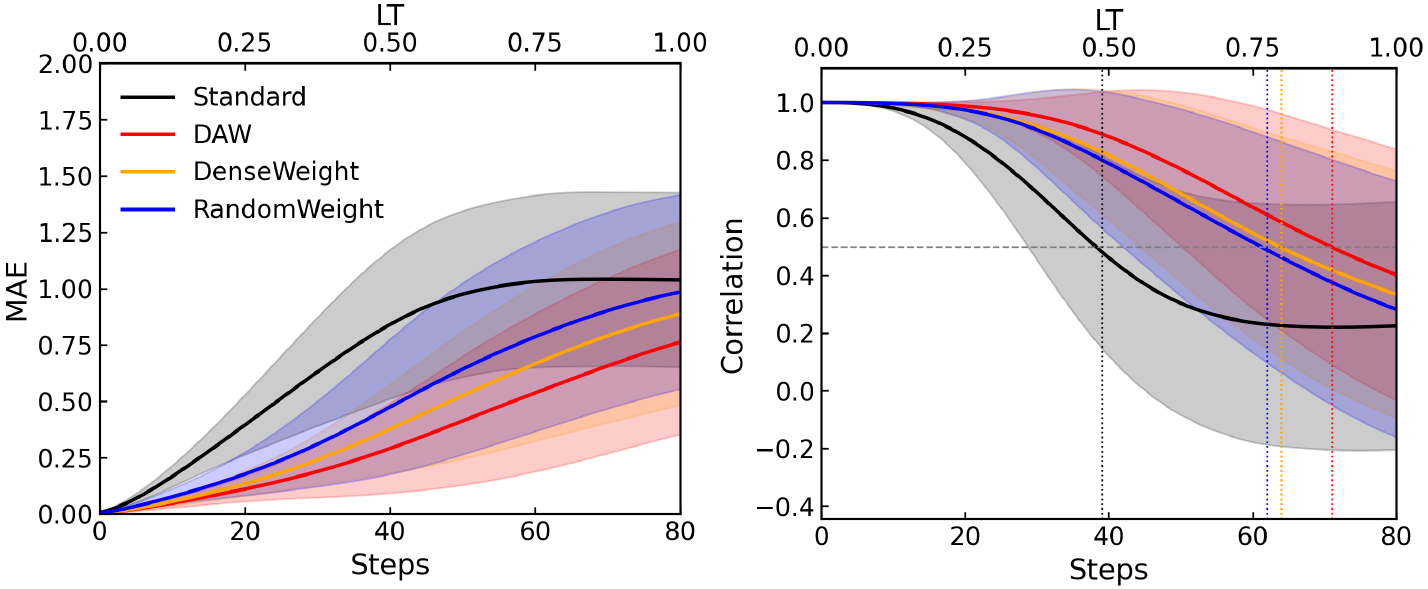}
\caption{\textbf{Autoregressive forecasting performance on the test set.} 
DAW (red) is compared against Standard (black), DenseWeight (orange), and RandomWeight (blue).
\textbf{Left:} Mean Absolute Error (MAE) at each autoregressive step, averaged over the full test set.
\textbf{Right:} Spatial Pearson Correlation Coefficient at each step, averaged over the test set.
In both panels, the primary horizontal axis (bottom) gives the discrete step index; the secondary axis (top) indicates the corresponding physical timescale in Lyapunov Time (LT).
Solid lines are per-step means across the test set; shaded bands denote $\pm 1$ standard deviation across test-set trajectories at each step.
In the correlation panel, a horizontal dashed grey line marks the $0.5$ predictability threshold; coloured vertical dotted lines indicate the step at which each model's mean correlation first crosses this threshold.}
\label{fig:autoregression-curve}
\end{figure}
Fig.~\ref{fig:autoregression-curve} traces this divergence over the full 1.0\,LT rollout.
In the MAE panel (left), all methods depart from near-zero error and progressively diverge; Standard exhibits the most rapid error accumulation, reaching a plateau near 1.0 by approximately 0.6\,LT, while DAW maintains the lowest MAE throughout.
The correlation panel (right) quantifies long-term spatial pattern fidelity against the $0.5$ predictability threshold, marked by the coloured dotted
lines.
Standard is the first to lose spatial coherence, crossing the threshold at approximately 0.5\,LT.
RandomWeight provides only a marginal delay to approximately 0.6\,LT.
DenseWeight preserves a high spatial correlation for a longer horizon, but still degrades progressively.
DAW is the last to cross the $0.5$ threshold, sustaining spatial coherence until approximately 0.85\,LT, revealed by the rightmost dotted line in
Fig.~\ref{fig:autoregression-curve} (right).
This ordering of methods is stable across the evaluated horizon and is consistent with the aggregate values in Table~\ref{tab:overall_evaluation_results}.

Moreover, while the main results use quantile $q=0.99$, we show in~Appendix~\ref{si:d_quantile} that DAW remains effective at $q=0.98$ and $q=0.995$, outperforming baselines in both cases. This indicates that DAW is robust to the choice of quantile.
\begin{figure}[H]
\centering
\includegraphics[width=0.9\linewidth]{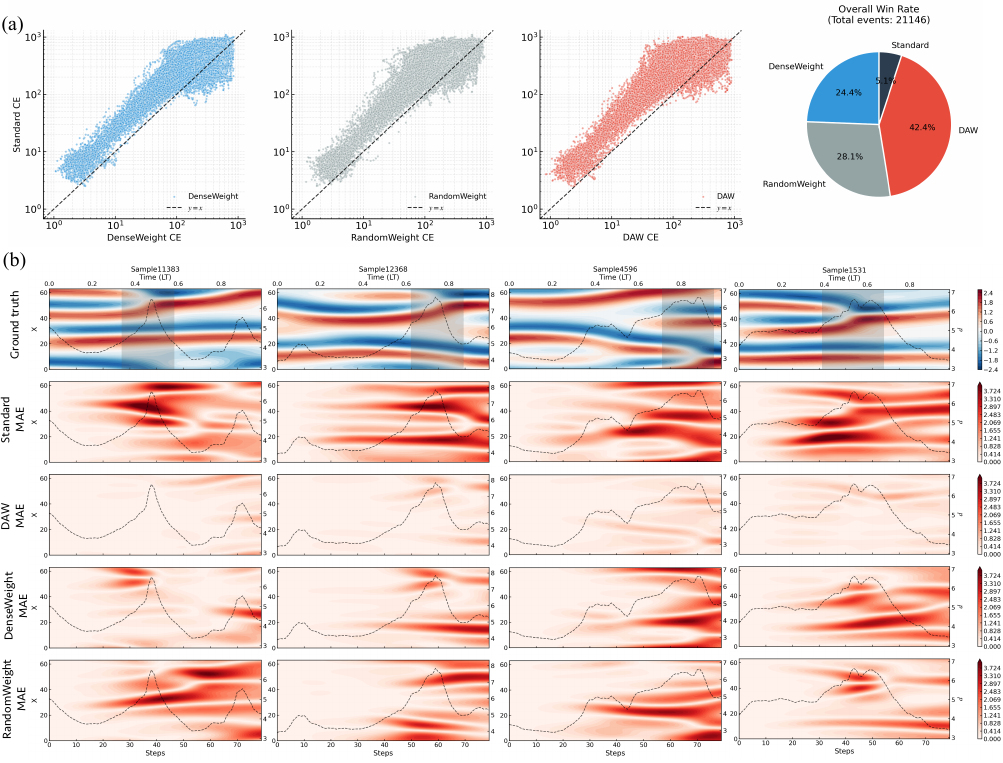}
\caption{\textbf{Event-level error analysis and spatiotemporal visualization of autoregressive forecasting errors on the KS test set.}
\textbf{(a)} \emph{Left three panels:} pairwise log-log scatter plots of the per-event Cumulative Error (CE), defined in Section~\ref{sec:evaluation_metrics} as the sum of absolute point-wise errors over a $\pm 3$-step window centered on the peak-$d$ time index of each identified high-$d$ event window.
Each scatter plot compares Standard CE ($y$-axis) against the CE of one competing method ($x$-axis); points \emph{above} the dashed $y = x$ diagonal indicate events where Standard incurs higher CE (the competing method wins), while points \emph{below} indicate events where Standard incurs lower CE.
\emph{Right panel:} overall win-rate pie chart reporting, across all 21{,}146 identified high-$d$ event windows, the fraction of windows in which each method achieves the lowest CE among all four competitors.
\textbf{(b)} Spatiotemporal MAE contour maps for four representative test trajectories (columns).
The top row shows the ground-truth KS field; subsequent rows show the spatial MAE of Standard, DAW, DenseWeight, and RandomWeight.
In all panels, the primary horizontal axis (bottom) gives the discrete step index and the secondary axis (top) gives the corresponding LT.
The primary vertical axis gives the spatial coordinate $x$.
Overlaid dashed curves trace the temporal evolution of the local dimension $d$ against the secondary vertical axis (right); grey-shaded intervals mark the identified high-$d$ event windows (exceeding the 75th-percentile threshold) at which CE is evaluated.}
\label{fig:sample_error}
\end{figure}

Fig.~\ref{fig:sample_error} provides a deeper look into how DAW mitigates the long-term error accumulation.
We generated rollout forecasts from 14{,}778 different initial conditions and have identified in total 21{,}146 high-$d$ event windows.
The pairwise scatter plots in Fig.~\ref{fig:sample_error}(a) compare, for each of these high-$d$ event windows, the CE of Standard against the ones of DenseWeight, RandomWeight, and DAW, respectively.
In the Standard-DAW panel (rightmost scatter), the large majority of events lie \emph{above} the $y = x$ diagonal.
This indicates that DAW systematically suppresses per-event error relative to the uniform objective, and that its advantage is not confined to a narrow slice of the CE distribution.
The overall win-rate pie chart confirms this observation. 
DAW achieves the lowest CE among all four methods in 42.4\% of the 21{,}146 event windows, far exceeding DenseWeight (24.4\%), RandomWeight (28.1\%), and Standard (5.1\%).
The statistical significance of the win rate is validated with bootstrapping, shown in~Appendix~\ref{si:bootstrap}.

The spatiotemporal panels in Fig.~\ref{fig:sample_error}(b) provide a physical perspective into how these event-level differences arise.
Across all four test trajectories, sharp spatial error jumps in Standard are both temporally localized and coincident with the grey-shaded high-$d$ windows.
These selected windows correspond to wave-merging events in the KS system.
Unlike simpler laminar flow and traveling waves, the wave-merging process is characterized by rapid steepening of localized spatial gradients and thus more complex dynamics.
Its onset (e.g. the merge of two pairs of red and blue wave packets) is captured by spikes in $d$, plotted with the black dotted lines.
As established by the decomposition in Eq.~\eqref{eq:decomp}, the cocycle $M_{k \to n}$ amplifies the one-step errors $\varepsilon_k$; since the largest one-step errors are committed at high-$d$ states, suppressing error there is a direct lever for slowing long-term error growth.
As shown in Fig.~\ref{fig:sample_error}(b), DAW suppresses error accumulation across all four samples by consistently mitigating the error jumps at the high-$d$ windows.
In contrast, DenseWeight and RandomWeight achieve only partial or inconsistent mitigation at these critical junctures, as documented by their win rates and aggregate metrics.

As a complementary analysis, we compare distributional consistency by pooling all test initial conditions at each fixed lead time and estimating the joint density of the first two spatial derivatives, $(u_{xx}, u_x)$. 
This projection is invariant under the translation symmetry of the periodic domain and thus provides a phase-independent characterization of the forecast dynamics.
We quantify the discrepancy between the forecast and reference measures using the sliced $2$-Wasserstein distance (WD), briefly introduced in~Appendix~\ref{si:attactor_method}. 
The results show that the relative ordering of the methods is preserved across all forecast horizons, with DAW consistently achieving the smallest WD for the joint distribution (Appendix~\ref{si:attractor_results}).

Moreover, the choice of factor space is likewise not critical. 
An ablation over the factor space -- input-only $d(x)$, target-only $d(y)$, and paired $d(x,y)$ (Appendix~\ref{si:dx_dy}) -- shows that $d(x)$ and $d(x,y)$ perform on par across the full rollout, and both are marginally ahead of $d(y)$, within run-to-run variability. 
Proposition~\ref{prop:lift} gives this equivalence for $d(x, y)$ versus $d(x)$; for $d(y)$, the targets $y_i = x_{i+1}$ sample the same invariant measure one step downstream, so all three indicators estimate the same local-dimension field asymptotically.
The slight disadvantage of $d(y)$ is a finite-scale effect rather than a difference in the information contained.

Additionally, our results show that DAW's advantage persists under two independent changes: the network architecture and the underlying system.
First, we replace the MLP with a 1D UNet backbone -- its convolutional, multi-scale receptive field encodes an inductive bias fundamentally different from the point-wise MLP.
The autoregressive forecast results are shown in Fig.~\ref{fig:ks-unet-rollout} (Appendix~\ref{si:ks_unet}), with DAW attaining the highest pattern correlation throughout the 1~LT horizon and the lowest MAE beyond approximately 0.5~LT.
Second, to test whether DAW's advantage reflects a general property of the reweighting rather than a feature of the KS geometry, we repeat the study on the Lorenz-63 system -- a low-dimensional chaotic attractor whose dynamics are largely distinct from the spatiotemporal KS field. 
Here the rare, high-$d$ excursions correspond to lobe-switching transitions and the edge of the attractor rather than wave-merging events, yet the same imbalance between recurrent bulk states and sparse transitions persists. 
As shown in Fig.~\ref{fig:lorenz_ar_metrics} (Appendix~\ref{si:daw_lorenz}), DAW attains the lowest rollout MAE and the slowest correlation decay over the four methods, followed by DenseWeight, RandomWeight, and Standard. 
Measured against the $0.5$ correlation threshold, DAW extends the predictability horizon from $\approx 2.3$~LT (Standard) to $\approx 4.1$~LT, with the purely statistical (DenseWeight) and shuffled-weight (RandomWeight) baselines falling in between.
This result is consistent with KS, indicating that DAW's benefit derives from aligning representational capacity with local dynamical complexity, rather than from any property specific to the KS system.

\section{Discussion}
\label{sec:discussion}
Recent advances in Scientific Machine Learning (SciML) have largely pursued model-centric innovations, including physics-informed neural networks, operator learning~\cite{li2021fourier, lu2021learning, wu2026differential}, and the architectural embedding of inductive biases~\cite{wu2022discovering}.
By contrast, the role of the empirical data distribution in shaping long-term forecast accuracy has received comparatively little attention.
The sample-wise uniform Mean Squared Error (MSE) objective treats all states as equally informative, driving the surrogate toward a statistical shortcut that fits the recurrent low-$d$ bulk of data while under-representing the rare, dynamically complex states that may trigger large rollout error accumulation.
Existing imbalanced-regression remedies calibrate the loss purely through observable target-space statistics, without leveraging sound physical insight.
Dynamics-Aware Weighting (DAW) reframes the problem as one of data-centric distribution reshaping rather than architecture design by using the local dimension $d$ as an \textit{a priori} measure of dynamical complexity. 
In this context, $d$ reshapes the optimization landscape toward the high-$d$ tail and forces the network to allocate representational capacity where forecast errors are systematically large.

Our control experiments indicate that the gain stems primarily from where weight is placed. 
RandomWeight, which applies non-uniform weights without regard to the dynamics, does improve on uniform training over autoregressive rollouts (Table~\ref{tab:overall_evaluation_results}) and at the event level (Fig.~\ref{fig:sample_error}), so non-uniform weighting alone already contributes; it nonetheless remains well behind DAW at every horizon, indicating that the alignment of weight with the local dimension supplies the larger share of the gain.
DenseWeight, which places weight by statistical rarity of the target alone, improves on uniform training but degrades faster than DAW at long autoregressive horizons (Section~\ref{sec:results}).
Target rarity is thus an unreliable proxy for chaotic dynamical systems; DAW's advantage comes from aligning the weight with the local dimension along the trajectory.
This alignment acts where the decomposition of Eq.~\eqref{eq:decomp} says it matters, and we stress that the theory of Section~\ref{subsec:daw-theory} motivates this placement rather than deriving the weighting rule itself.
By reducing one-step errors at high-$d$ states, DAW shrinks the terms that the cocycle subsequently amplifies, and thereby suppresses the localized error jumps that occur when the trajectory enters high-$d$ regimes.
The net effect is a systematic delay in the loss of spatial coherence and improved long-term accuracy relative to both uniform training and purely density-based weighting, not a removal of the intrinsic limits of chaotic predictability.

A practical strength of DAW is its plug-and-play nature.
Because $d$ is computed offline from the historical trajectory, the framework introduces no architectural change and no inference overhead, leaving the trained surrogate as lightweight as its unweighted counterpart while curbing rollout error growth and preserving spatial-pattern consistency.

In this work, we restrict our experiments to the KS and Lorenz-63 systems.
While the DAW objective is agnostic to the state dimension and to the network architecture, the cost of robustly estimating the \textit{a priori} signal $d$ scales with the \emph{effective} dimension of the attractor. 
A reliable GPD fit requires a trajectory whose length grows with that effective dimension, so care is needed when scaling to higher-dimensional systems.

\section*{Data and code availability}
\noindent The code and data used in this study are publicly available at
\url{https://github.com/MathEXLab/DAW}.

\begin{appendices}

\renewcommand{\thefigure}{\thesection\arabic{figure}}
\renewcommand{\thetable}{\thesection\arabic{table}}

\makeatletter
\@addtoreset{figure}{section}
\@addtoreset{table}{section}
\makeatother

\section{Attractor-level comparison in the $(u_{xx}, u_x)$ plane for KS}\label{si:ks_attractor}
The metrics of Section~\ref{sec:results} are trajectory-wise: they compare a forecast against the reference orbit it was initialised from, and are therefore bounded by the intrinsic predictability horizon of the system. A complementary question is whether the surrogate reproduces the invariant measure of the attractor.

\subsection{Method}\label{si:attactor_method}
Let $\mu$ denote the invariant measure of the KS attractor and $\hat f$ the learned one-step map; the states reached after $k$ autoregressive steps follow $\nu_k = (\hat f^{\,k})_* \mu$. We estimate this pushforward measure by pooling all $14{,}778$ test initial conditions at a fixed step index $k$ and repeating the construction at several different forecasting horizons. 

The KS state is a single scalar field defined over a periodic domain of size $L=22$ (discretized into $N_x=64$ grid points), thus we project the phase space onto the first two spatial derivatives $(u_{xx}, u_x)$: both are equivariant under the translation symmetry of the periodic domain, providing a phase-independent characterization of the forecast dynamics. Furthermore, both enter the energy budget $\dot E = \langle u_x^2\rangle - \langle u_{xx}^2\rangle$, and both register the steep localized gradients that accompany dynamically complex wave-merging events. Derivatives are evaluated spectrally,
\begin{equation}
u^{(m)}(x_j) = \sum_{|n| \le n_c} (\mathrm{i}k_n)^m\, \hat u_n\, e^{\mathrm{i}k_n x_j},
\qquad m \in \{1,2\},
\label{eq:spectral-derivative}
\end{equation}
with $k_n = 2\pi n / L$ and $\hat u_n$ the Fourier coefficients of the physical (de-standardised) field. The band limit $n_c = 16$ is a necessary numerical regularisation rather than a cosmetic choice: the multiplier $k_n^2$ amplifies the unresolved high-wavenumber tail by more than two orders of magnitude, whereas the KS spectrum at $L = 22$ has already decayed exponentially by $n \ge 10$. It is applied identically to the reference data and to every surrogate.

The statistical discrepancy between $\nu_k$ and $\mu$ is quantified by the sliced $2$-Wasserstein distance~\cite{bonneel2015sliced}. 
For two densities $\rho$ and $\rho_{\mathrm{ref}}$ on the $(u_{xx}, u_x)$ plane,
\begin{equation}
\mathrm{WD}^2(\rho, \rho_{\mathrm{ref}})
= \int_{\mathbb{S}^1} W_2^2\!\left(\pi_{\theta}\rho,\; \pi_{\theta}\rho_{\mathrm{ref}}\right) \mathrm{d}\sigma(\theta),
\label{eq:sliced-wasserstein}
\end{equation}
where $\pi_\theta$ projects onto the unit direction $\theta$ and $\sigma$ is the uniform measure on the circle; the integral is numerically approximated with $200$ directions distributed uniformly on the half-circle. Both phase-space coordinates are normalised by the standard deviations of the reference measure, so that $\mathrm{WD}$ is dimensionless and directly comparable across different lead times.

\subsection{Results}\label{si:attractor_results}
\begin{figure}[H]
    \centering
    \includegraphics[width=0.9\linewidth]{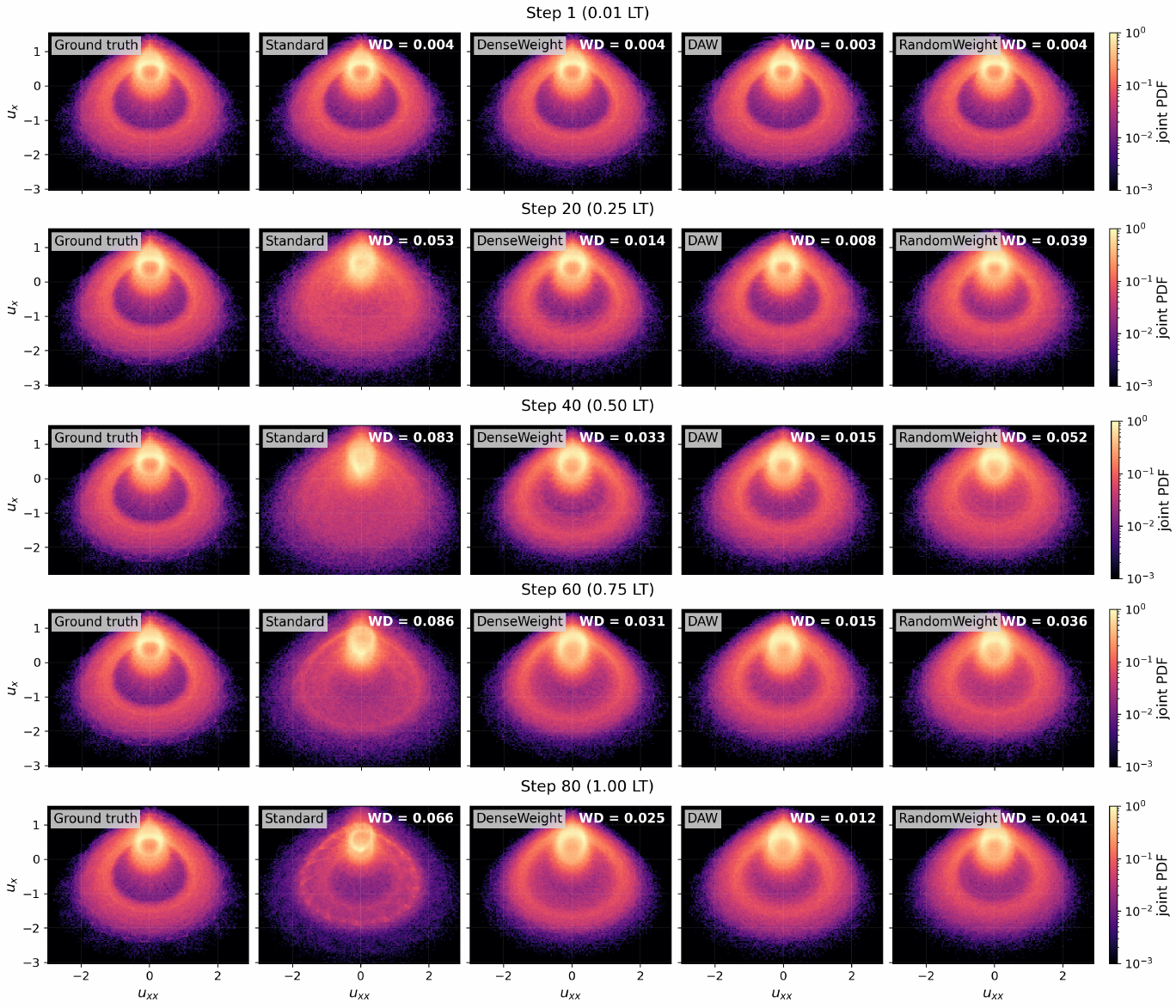}
    \caption{\textbf{Joint probability density of $(u_{xx}, u_x)$ along the
    autoregressive rollout.} Rows correspond to the lead times $k = 1$, $20$, $40$,
    $60$ and $80$ steps ($0.01$ to $1.0$\,LT); columns give the reference data
    followed by the four training objectives. Each panel is the density of the
    pushforward measure $\nu_k$, estimated by pooling all test initial conditions at
    the fixed step index $k$, with derivatives computed spectrally under the band
    limit $n_c = 16$ applied identically to the reference data and to every
    surrogate. All panels share the same bin edges and the same logarithmic colour
    scale, normalised to unit peak density. The annotation in each panel reports the
    sliced $2$-Wasserstein distance (WD) to the reference measure at the same lead
    time, defined in Eq.~\eqref{eq:sliced-wasserstein}. The four objectives are
    indistinguishable at $k = 1$ and separate monotonically thereafter, with DAW
    holding the smallest distance at every subsequent horizon.}
    \label{fig:joint-pdf}
\end{figure}
Fig.~\ref{fig:joint-pdf} reports the joint density at five lead times. The reference density is symmetric about $u_{xx} = 0$, as required by the equivariance of the KS equation under $u(x) \mapsto -u(-x)$. 
At $k = 1$ the four objectives are indistinguishable, with $\mathrm{WD}$ between $0.003$ and $0.004$, consistent with the near-parity of the single-step metrics in Table~\ref{tab:overall_evaluation_results}. 
The methods separate from $0.25$\,LT onward and hold a fixed ordering at every subsequent horizon: DAW attains the smallest distance ($0.008$, $0.015$, $0.015$ and $0.012$ at $0.25$, $0.5$, $0.75$ and $1.0$\,LT), followed by DenseWeight, RandomWeight and Standard, the last reaching $0.066$ at $1.0$\,LT -- larger than DAW by a factor of five.
Unlike the MAE, $\mathrm{WD}$ is not monotone in $k$: all four methods peak near $0.5$--$0.75$\,LT and decline slightly thereafter, as the ensemble relaxes onto the invariant measure of the surrogate map itself.

\begin{figure}[H]
    \centering
    \includegraphics[width=0.8\linewidth]{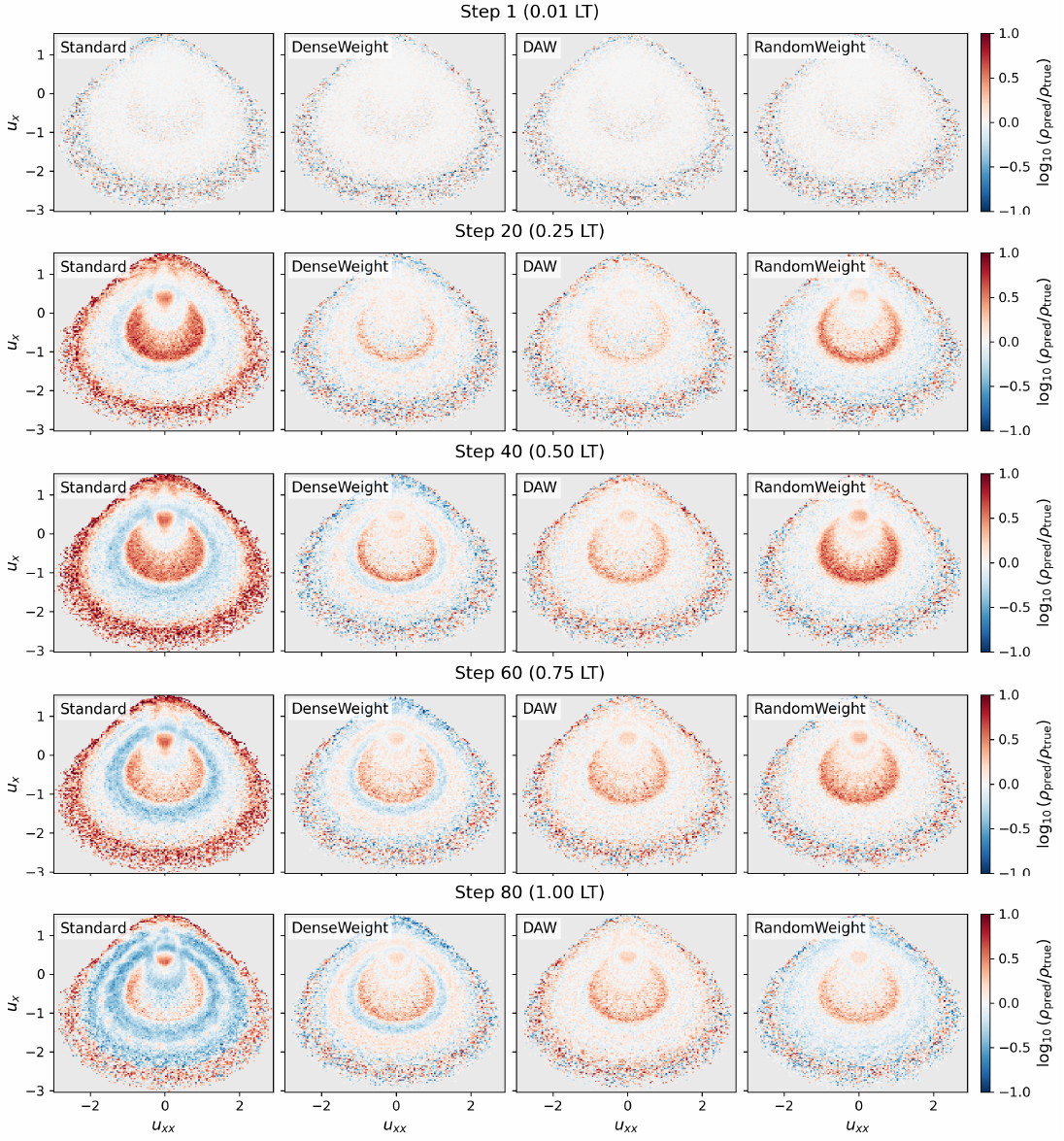}
    \caption{\textbf{Density ratio relative to the reference measure.} Base-ten
    logarithm of the ratio $\rho_{\mathrm{pred}}/\rho_{\mathrm{true}}$ on the bin
    grid of Fig.~\ref{fig:joint-pdf}, with rows and columns as in that figure and
    the reference column omitted. Red indicates an excess of probability mass
    relative to the reference and blue a deficit; bins that are empty in either
    density are left blank. The colour scale is symmetric about zero and shared
    across all panels. Whereas Fig.~\ref{fig:joint-pdf} reports how far each
    measure has drifted, this figure localises the drift: under uniform training the
    excess mass accumulates in the interior of the reference distribution while the
    outer arcs are depleted, the signature of a map that has smoothed the sharp
    gradient structure of the attractor.}
    \label{fig:joint-pdf-differences}
\end{figure}

Fig.~\ref{fig:joint-pdf-differences} localises the differences in the joint distribution. 
At $k = 1$ the four ratio maps are featureless. From $0.25$\,LT the Standard panels develop coherent structure of order one decade: probability mass accumulates in the low-density interior of the arc system while the arcs themselves are depleted. RandomWeight retains a weaker version of the same pattern and DenseWeight weaker still, whereas the DAW panels remain close to zero throughout, with no coherent structure beyond bin-level noise. The departure under uniform training is therefore concentrated in the steep-gradient configurations that Section~\ref{sec:results} identifies, through the $d$-binned analysis, as the site of the largest one-step errors.

%
\section{Convergence and consistency of local dimension $d$}\label{si:d_convergence}
\begin{figure}[h]
\centering
\includegraphics[width=0.8\linewidth]{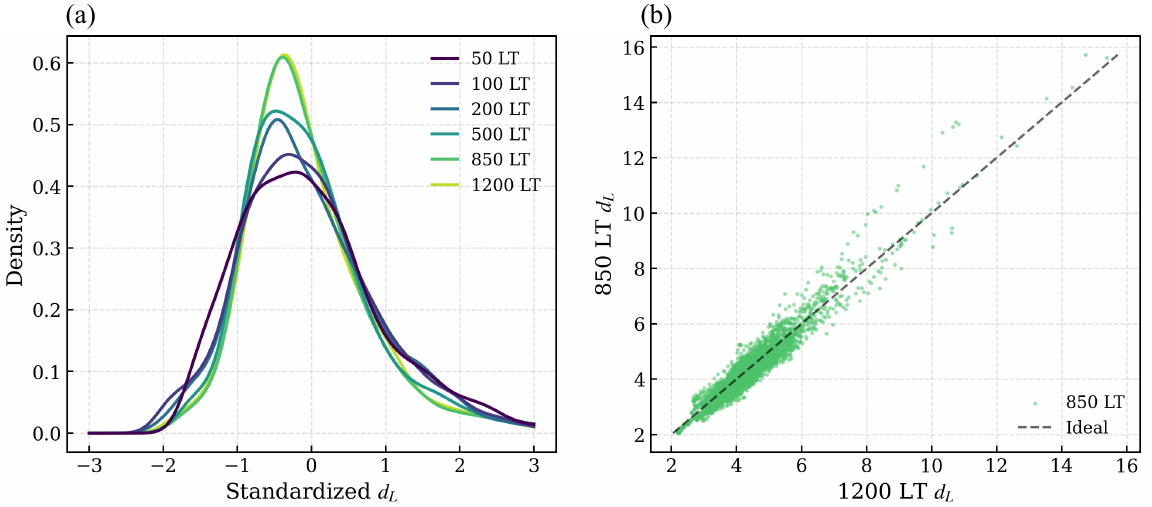}
\caption{\textbf{Convergence and Robustness of the $d$ Estimation.} 
(a) Probability density functions of the standardized local dimension $d$ estimated from trajectories of increasing length (from 50 LT to 1200 LT); in the axis labels, $d_L$ denotes $d$ estimated from a trajectory of the indicated length.
(b) Consistency comparison of the $d$ estimated at 850 LT versus 1200 LT. 
A perfect fit is shown with the dashed diagonal line.}
\label{fig:d_convergence}
\end{figure}
To ensure the reliability of the local dimension ($d$) as a dynamical indicator for the DAW framework, we validated the trajectory length required for a robust estimation. 
Fig.~\ref{fig:d_convergence} illustrates the convergence and consistency of $d$ across varying trajectory lengths. 
Panel (a) shows the probability density functions of the local dimension $d$ estimated from trajectories of increasing length, under Z-score normalization, $(d - \mu_{d})/\sigma_{d}$.
It exhibits significant fluctuations at shorter sequence lengths (e.g., 50 to 200 LT). 
However, as the observation length increases, the distribution systematically converges under Z-score normalization, stabilizing effectively at 850 LT. 
To further verify this consistency of magnitude across samples, we conducted a point-wise comparison of the $d$ estimates obtained using 850 LT against a longer 1200 LT baseline, as shown in Panel (b). 
The strong linear consistency observed between the two configurations confirms that extending the sequence beyond 850 LT yields negligible changes to the relative dynamical complexity assigned to each state. Consequently, the 850 LT configuration adopted throughout our experiments robustly extracts local dimensions, which is necessary for accurate density-based weighting.

\section{Sensitivity of $\alpha$}\label{si:sensitivity_alpha}
\begin{figure}[ht]
\centering
\includegraphics[width=0.9\linewidth]{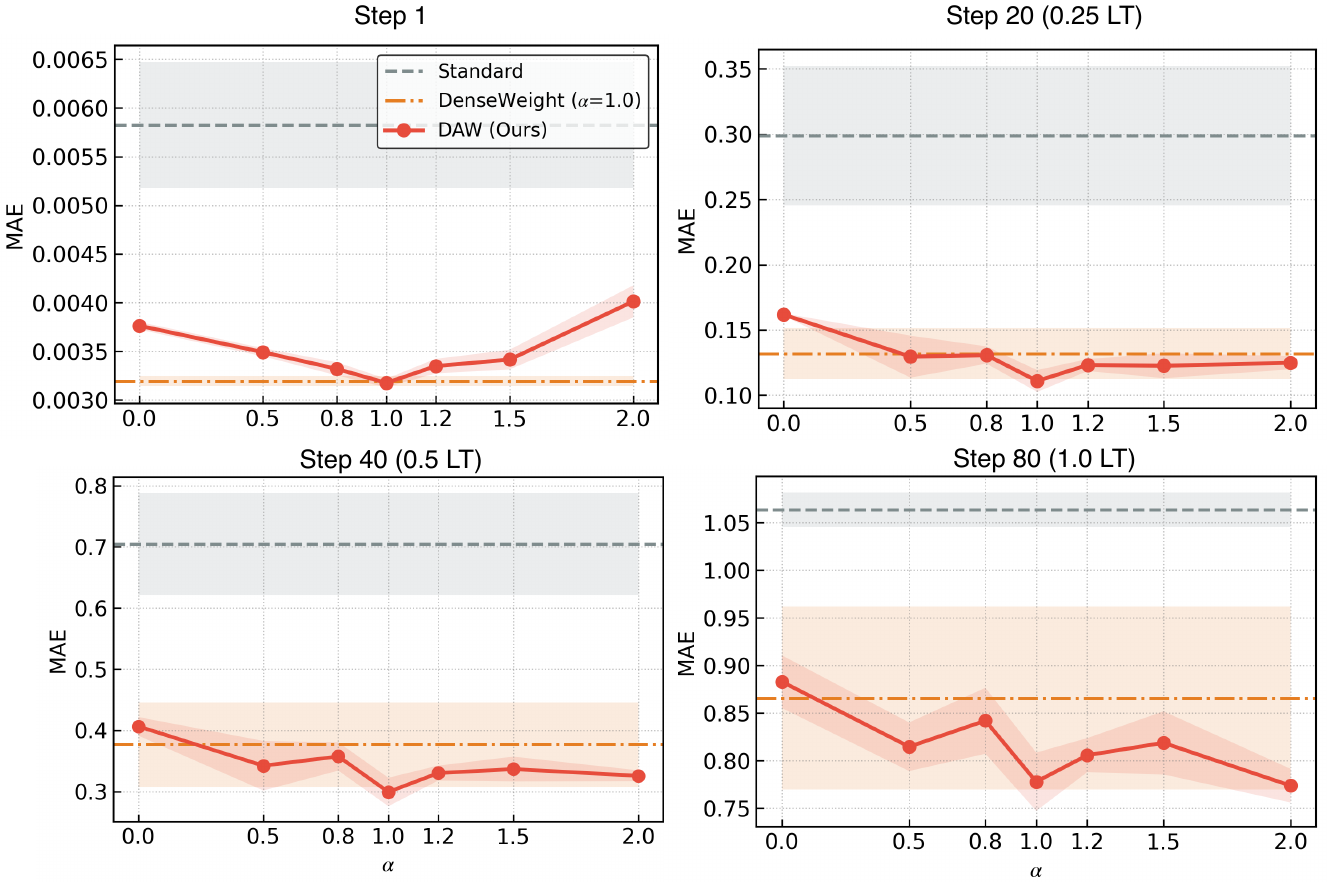}
\caption{Hyperparameter sensitivity analysis of the weighting intensity parameter $\alpha$. 
The plots display the Mean Absolute Error (MAE) evaluated at four autoregressive forecasting horizons: Step 1, Step 20 (0.25 LT), Step 40 (0.5 LT), and Step 80 (1.0 LT). 
The red solid line represents the DAW model (mean $\pm$ std over three independent runs). 
The horizontal grey dashed line and orange dash-dotted line indicate the mean performance of the Standard baseline and the DenseWeight ($\alpha=1.0$) baseline, respectively; shaded bands around each line denote $\pm 1$ standard deviation over the three independent runs.
At extended forecast horizons (0.5 LT and 1.0 LT), DAW demonstrates a broad and robust performance basin, consistently outperforming both baselines across a wide range of $\alpha$ values, confirming that its efficacy is not contingent upon narrow hyperparameter tuning.}
\label{fig:param_sensitivity}   
\end{figure}
To rigorously evaluate the robustness of the proposed Dynamics-Aware Weighting (DAW) framework, we conduct a hyperparameter sensitivity analysis on the weighting intensity parameter, $\alpha$. 
This parameter controls how strongly densely sampled states are down-weighted relative to the sparse high-$d$ (dynamically complex) tail. 
We evaluate the model's performance across a broad spectrum of intensities, ranging from pure linear $\tilde{d}$-tilt ($\alpha \approx 0.0$) to aggressive penalization ($\alpha = 2.0$).

Fig.~\ref{fig:param_sensitivity} illustrates the evolution of the Mean Absolute Error (MAE) at four distinct autoregressive forecasting horizons: Step 1, Step 20 (0.25 LT), Step 40 (0.5 LT), and Step 80 (1.0 LT). 
The horizontal lines denote the performance of the unweighted Standard baseline (grey dashed) and the DenseWeight ($\alpha = 1.0$) baseline (orange dash-dotted). 

At the initial forecast step (Step 1), the predictive errors are universally low, and variations in $\alpha$ yield negligible differences. 
However, as the forecast horizon extends to 0.5 LT and 1.0 LT, the critical importance of dynamics-aware calibration becomes evident.
As expected, when $\alpha \to 0$, the density term disengages and DAW reduces to the pure $\tilde{d}$-tilt of Eq.~\eqref{eq:wreshaped}, which improves on uniform training but accumulates error faster than the full weighting.
Conversely, as $\alpha$ increases, the MAE drops significantly, forming a robust performance basin. 

Crucially, across a wide range (from $\alpha = 0.5$ to $\alpha = 2.0$), the DAW method consistently outperforms both the Standard and DenseWeight baselines at long forecast horizons.
The optimal performance is observed around $\alpha = 1.0$. 
The existence of this broad and stable basin demonstrates that DAW is highly robust; its superior long-term stability is fundamentally rooted in its phase-space topological alignment, rather than relying on brittle, fine-tuned hyperparameter artifacts.

\section{Comparison of using $d(x,y)$, $d(x)$ and $d(y)$}\label{si:dx_dy}

In the main text, the local dimension is estimated on the one-step map: for each sample we concatenate the input--output pair $(x_i, y_i)$ into a higher-dimensional phase space and compute $d(x,y)$.
Proposition~\ref{prop:lift} shows that, in the vanishing-radius limit, this pairing leaves the local dimension unchanged, so that $d(x,y)$ and the state-only indicator $d(x)$ coincide asymptotically; at the finite quantile $q$ used in practice, however, the paired and state-only estimates need not agree. To identify which factor space the weighting should be built on, we recompute $d$ from the input state alone ($d(x)$) and from the target state alone ($d(y)$), holding the DAW construction of Section~\ref{subsec:daw} and all training settings fixed, and
varying only the indicator that enters the density estimate. The autoregressive results are reported in Table~\ref{tab:dxy_evaluation_results}.

Two observations follow. First, $d(x)$ and $d(x,y)$ are statistically indistinguishable across the full rollout: their MAE and spatial correlation agree to within run-to-run variability at every horizon, from the single-step prediction
(Step~1) through $1.0$~LT (Step~80). This is consistent with
Proposition~\ref{prop:lift} -- once the recurrence geometry of the input state is retained, the target component adds no separable information at the resolved scales. 
Second, both the paired and the input-only indicators outperform
the target-only $d(y)$, although the margin remains within run-to-run variability at every horizon.
At $1.0$~LT, $d(x,y)$ and $d(x)$ reach an MAE of $0.778 \pm 0.031$ and $0.779 \pm 0.017$ against $0.794 \pm 0.015$ for $d(y)$, and a preserved correlation of $0.415 \pm 0.035$ and $0.417 \pm 0.017$ against $0.393 \pm 0.021$.
The recurrence geometry of the input state thus carries essentially all of the signal that DAW exploits, whereas the target-only indicator is consistently, if marginally, the weakest of the three.

These results support the choice adopted in the main text and, at the same time,
indicate that DAW is not sensitive to it. Because $d(x)$ recovers the performance
of $d(x,y)$ while requiring only the input state, it constitutes an equally valid
and fully \emph{a priori} construction -- the indicator can be obtained from the
historical trajectory without reference to the one-step target, so the reweighting
remains a function of the state recurrence dimension alone.

\begin{table}[ht]
    \centering
    \caption{\textbf{Factor-space ablation results} (mean $\pm$ std over three independent runs). 
    The best-performing method for each metric at every forecast horizon is highlighted in \textbf{bold}.\\}
    \label{tab:dxy_evaluation_results}
    \begin{tabular}{clcc}
        \toprule
        \textbf{Step} & \textbf{Method}
        & \textbf{MAE} $\downarrow$
        & \textbf{Correlation} $\uparrow$ \\
        \midrule
        & $d(x)$        & 0.0033 $\pm$ 0.0001
                        & \textbf{1.0000 $\pm$ 0.0000} \\
        & $d(x,y)$      & \textbf{0.0032 $\pm$ 0.0000}
                        & \textbf{1.0000 $\pm$ 0.0000} \\
                        
        \multirow{-3}{*}{1}
        & $d(y)$        & 0.0033 $\pm$ 0.0001
                        & \textbf{1.0000 $\pm$ 0.0000} \\
                        \midrule
         
        & $d(x)$        & 0.1171 $\pm$ 0.0021
                        & 0.9867 $\pm$ 0.0001 \\
         
        & $d(x,y)$      & \textbf{0.1109 $\pm$ 0.0084}
                        & \textbf{0.9876 $\pm$ 0.0020} \\
         
        \multirow{-3}{*}{20 (0.25\,LT)}
        & $d(y)$        & 0.1210 $\pm$ 0.0059
                        & 0.9856 $\pm$ 0.0010 \\
                        \midrule
        & $d(x)$        & 0.3085 $\pm$ 0.0070
                        & \textbf{0.8875 $\pm$ 0.0041} \\
        & $d(x,y)$      & \textbf{0.2993 $\pm$ 0.0233}
                        & 0.8874 $\pm$ 0.0170 \\
        \multirow{-3}{*}{40 (0.5\,LT)}
        & $d(y)$        & 0.3191 $\pm$ 0.0180
                        & 0.8741 $\pm$ 0.0137 \\
                        \midrule
         
        & $d(x)$        & 0.7793 $\pm$ 0.0174
                        & \textbf{0.4174 $\pm$ 0.0173} \\
         
        & $d(x,y)$      & \textbf{0.7779 $\pm$ 0.0306}
                        & 0.4153 $\pm$ 0.0353 \\
         
        \multirow{-3}{*}{80 (1.0\,LT)}
        & $d(y)$        & 0.7939 $\pm$ 0.0146
                        & 0.3930 $\pm$ 0.0210 \\
        \bottomrule
    \end{tabular}
\end{table}

\section{Comparison of $d$ quantile}\label{si:d_quantile}
Beyond the weighting intensity $\alpha$, the quantile $q$ is the principal free parameter in the estimation of the local dimension $d$. 
To verify that DAW does not rely on a finely tuned threshold, we evaluate its performance under three settings, $q \in \{0.98,\, 0.99(\text{main text}),\, 0.995\}$.

By construction, the Standard (uniform) and DenseWeight (target-space density) baselines are independent of $q$, we therefore retain the Standard and DenseWeight results from the main text and retrain only DAW and RandomWeight, recomputing $d$ at $q = 0.98$ and $q = 0.995$ while holding all remaining settings fixed.

Fig.~\ref{fig:daw_0.995} and Fig.~\ref{fig:daw_0.98} report the autoregressive rollouts for $q = 0.995$ and $q = 0.98$, respectively. 
Across both settings, DAW attains the lowest MAE and the slowest decay of spatial correlation, delaying the crossing of the $0.5$ predictability threshold the longest.
This result indicates that the performance gain of DAW does not rely on a fine-tuned threshold, and the estimate of $d$ is sufficiently robust over the tested range of $q$ to leave the relative ranking of dynamically complex states, and hence the induced weighting essentially unchanged.

\begin{figure}[H]
    \centering
    \includegraphics[width=0.8\linewidth]{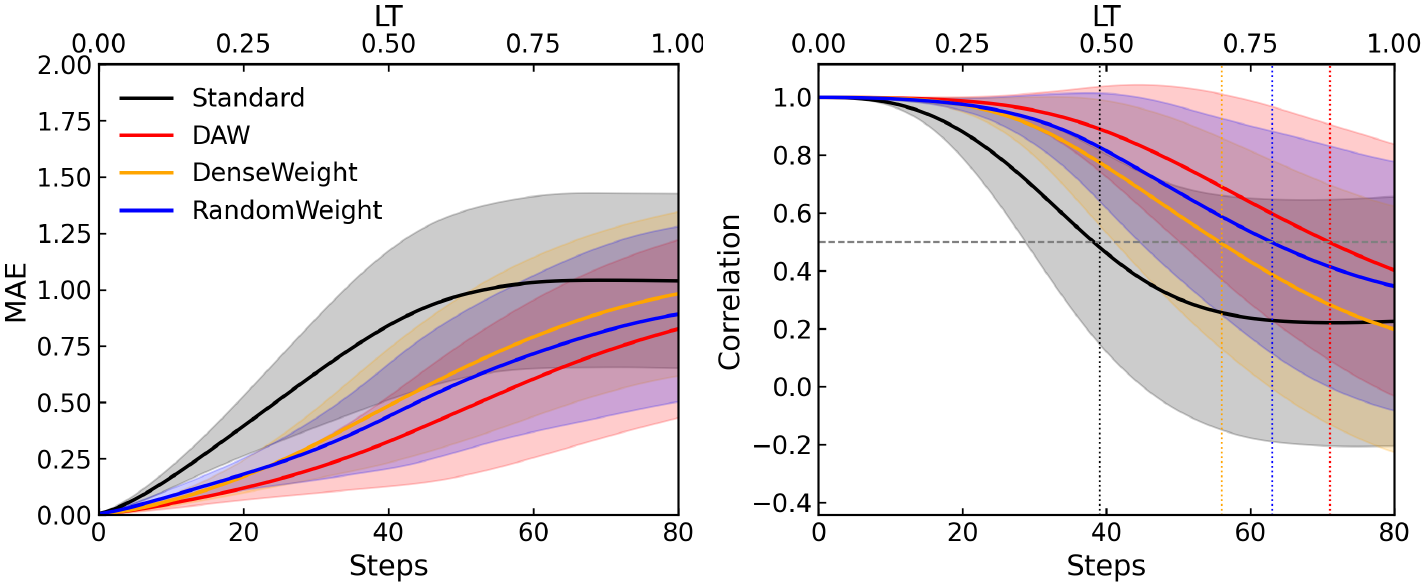}
    \caption{\textbf{Autoregressive forecasting performance on the test set ($q=0.995$).} 
DAW (red) is compared against Standard (black), DenseWeight (orange), and RandomWeight (blue).
\textbf{Left:} Mean Absolute Error (MAE) at each autoregressive step, averaged over the full test set.
\textbf{Right:} Spatial Pearson Correlation Coefficient at each step, averaged over the test set.
In both panels, the primary horizontal axis (bottom) gives the discrete step index; the secondary axis (top) indicates the corresponding physical timescale in Lyapunov Time (LT).
Solid lines are per-step means across the test set; shaded bands denote $\pm 1$ standard deviation across test-set trajectories at each step.
In the correlation panel, a horizontal dashed grey line marks the $0.5$ predictability threshold; coloured vertical dotted lines indicate the step at which each model's mean correlation first crosses this threshold.}
    \label{fig:daw_0.995}
\end{figure}

\begin{figure}[H]
    \centering
    \includegraphics[width=0.8\linewidth]{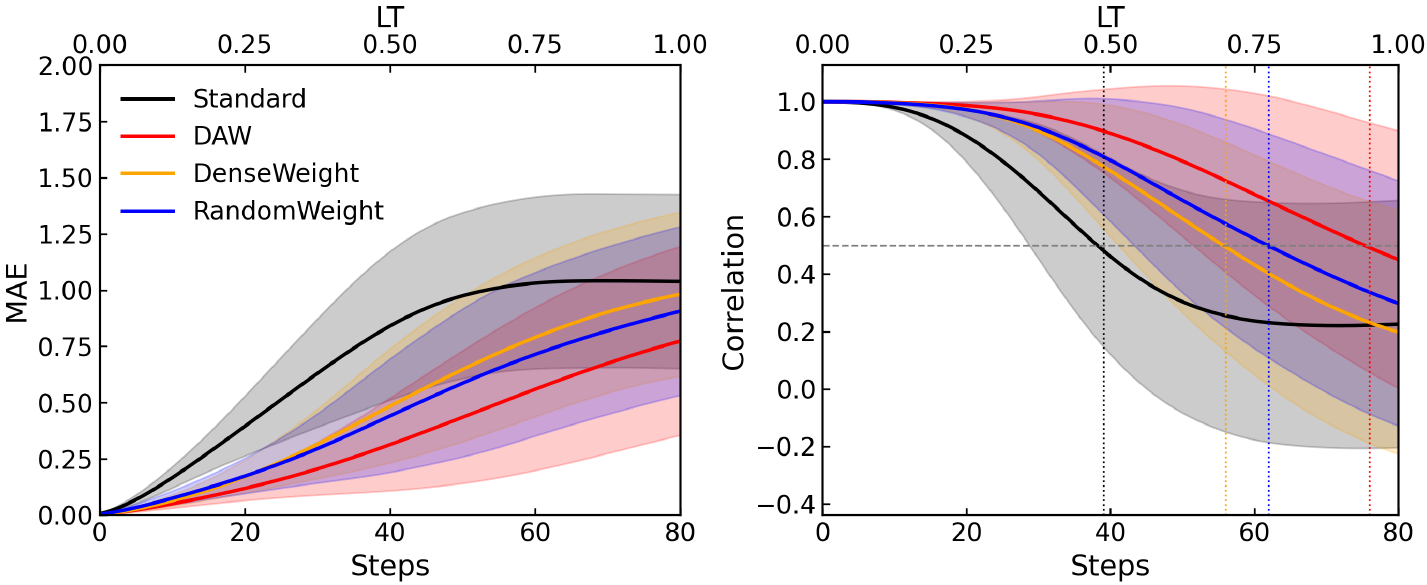}
    \caption{\textbf{Autoregressive forecasting performance on the test set ($q=0.98$).} 
DAW (red) is compared against Standard (black), DenseWeight (orange), and RandomWeight (blue).
\textbf{Left:} Mean Absolute Error (MAE) at each autoregressive step, averaged over the full test set.
\textbf{Right:} Spatial Pearson Correlation Coefficient at each step, averaged over the test set.
In both panels, the primary horizontal axis (bottom) gives the discrete step index; the secondary axis (top) indicates the corresponding physical timescale in Lyapunov Time (LT).
Solid lines are per-step means across the test set; shaded bands denote $\pm 1$ standard deviation across test-set trajectories at each step.
In the correlation panel, a horizontal dashed grey line marks the $0.5$ predictability threshold; coloured vertical dotted lines indicate the step at which each model's mean correlation first crosses this threshold.}
    \label{fig:daw_0.98}
\end{figure}

\section{Additional results on KS with UNet} \label{si:ks_unet}
To verify that the advantage of DAW is not tied to a particular surrogate architecture, we repeat the KS study with a convolutional U-Net backbone in place of the MLP used in the main text.
\subsection{Experimental setup}
\label{si:ks-unet-setup}

The data pipeline, input-output length, training scheme, and weighting schemes are identical to Section~\ref{sec:data}--\ref{sec:baselines}, with the only change being the surrogate backbone. 
We use a one-dimensional convolutional U-Net (1D UNet) that maps the inputs fields to the one-step-ahead field. 
It follows a symmetric encoder--decoder design: a double-convolution stem lifts the $3$ input channels (corresponding to the stacked input snapshots) to a base width $C=32$, three encoder stages each apply max-pooling by two then a double convolution ($32 \rightarrow 64 \rightarrow 128 \rightarrow 256$ channels, spatial length $64 \rightarrow 32 \rightarrow 16 \rightarrow 8$), a mirrored decoder restores resolution through three transposed-convolution stages with skip connections, and a final $1\times1$ convolution projects back to the single output field; every double convolution is two repetitions of a width-$3$ 1D convolution (padding $1$, no bias), batch normalization, and ReLU. Relative to the point-wise MLP, this backbone carries a different inductive bias through local convolutions and multi-scale skip connections, providing a test of whether the benefit of DAW transfers across architectures.

\subsection{Results}
\label{si:ks-unet-results}
\begin{figure}[ht]
    \centering
    \includegraphics[width=0.8\linewidth]{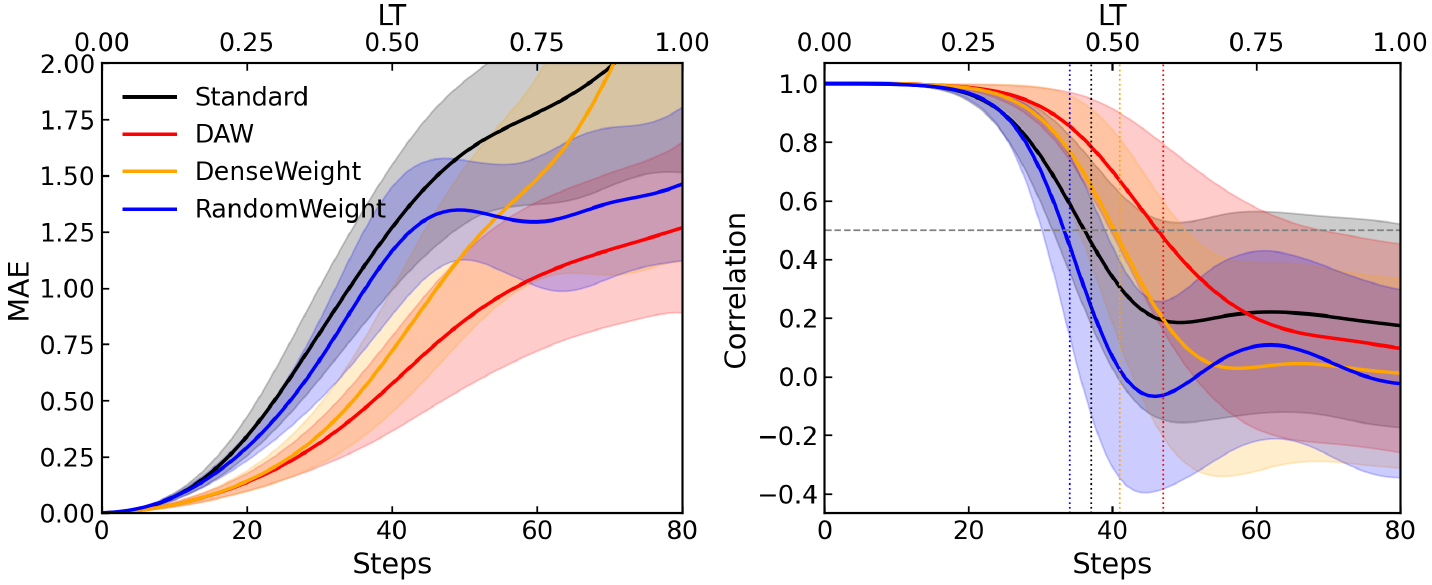}
    \caption{\textbf{Autoregressive forecasting performance on the KS test set with a 1D UNet backbone.} DAW (red) is compared against Standard (black), DenseWeight (orange), and RandomWeight (blue). \textbf{Left:} Mean Absolute Error (MAE) at each autoregressive step, averaged over the full test set. \textbf{Right:} Spatial Pearson Correlation Coefficient at each step, averaged over the test set. In both panels, the primary horizontal axis (bottom) gives the discrete step index; the secondary axis (top) indicates the corresponding physical timescale in Lyapunov Time (LT). Solid lines are per-step means across the test set; shaded bands denote $\pm 1$ standard deviation across test-set trajectories at each step. In the correlation panel, a horizontal dashed grey line marks the $0.5$ predictability threshold; coloured vertical dotted lines indicate the step at which each model's mean correlation first crosses this threshold.}
    \label{fig:ks-unet-rollout}
\end{figure}

Fig.~\ref{fig:ks-unet-rollout} reports the closed-loop rollout under the 1D UNet backbone.
DAW attains the lowest MAE beyond $\approx$0.5\,LT, with DenseWeight comparable at intermediate lead times before its error grows without bound, and is the last to lose spatial coherence, crossing the $0.5$ correlation threshold at approximately $0.6$~LT, against approximately $0.5$~LT for DenseWeight, and shorter than $0.5$~LT for RandomWeight and Standard.
Because the UNet and the MLP carry distinct inductive biases, the persistence of DAW's advantage under a different inductive bias indicates that the benefit of DAW derives from reshaping the effective training measure toward dynamically complex states, rather than an artifact confined to a particular network class.

\newpage
\section{Additional results on Lorenz}\label{si:daw_lorenz}
To test whether DAW's advantage is specific to spatiotemporal chaos, we repeat the study on the Lorenz-63 system -- a low-dimensional chaotic attractor whose geometry is largely distinct from the KS field.

\subsection{Experimental setup} \label{si:lorenz-setup}
The Lorenz-63 system is governed by
\begin{equation}
    \dot{x} = \sigma(y - x), \qquad
    \dot{y} = x(\rho - z) - y, \qquad
    \dot{z} = xy - \beta z,
    \label{eq:lorenz63}
\end{equation}
with the canonical parameters $(\sigma,\rho,\beta) = (10,\,28,\,8/3)$, for which the system is chaotic with leading Lyapunov exponent $\lambda_{\max}\approx 0.906$ (Lyapunov time $\mathrm{LT}=1/\lambda_{\max}\approx 1.10$).

Trajectories are integrated with a fourth-order Runge-Kutta scheme at $dt = 0.01$. 
The total trajectory spans around 5000 LT, with a 70\%/15\%/15\% training/validating/testing split.
All data is standardized via Z-score normalization using training-split statistics.

The local dimension $d$ is estimated with the identical pipeline as the KS, applied here to the three-dimensional state; the rare, high-$d$ excursions correspond to lobe-switching transitions between the two wings of the attractor -- the low-dimensional analogue of the KS wave-merging events. 
For the Lorenz system, we adopt a simpler MLP with 4 hidden layers and 64 neurons in each hidden layer.
The remaining configurations, and the four weighting schemes (Standard, DAW, DenseWeight, RandomWeight) are identical to Section~\ref{sec:training}.

\subsection{Results}\label{si:lorenz_results}
\begin{figure}[htbp]
\centering
\includegraphics[width=0.8\linewidth]{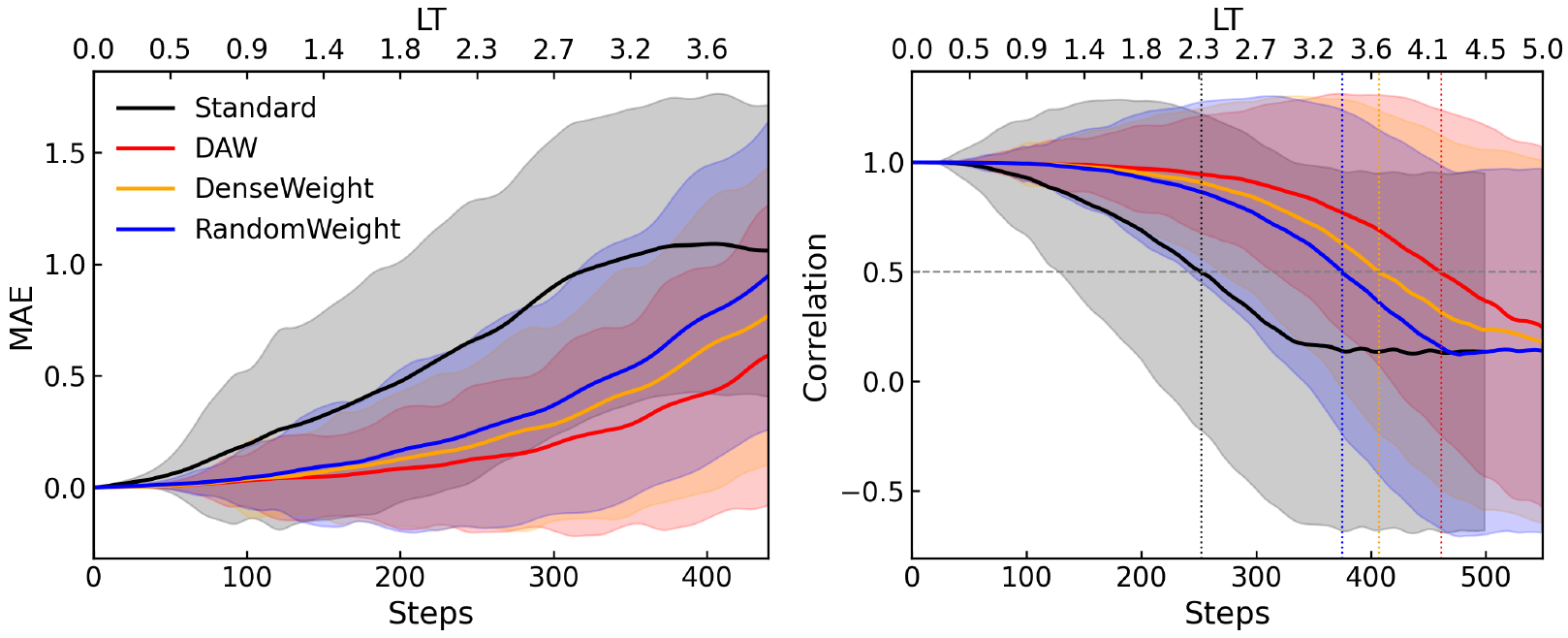}
\caption{\textbf{Autoregressive forecasting performance on the test set for Lorenz.} 
DAW (red) is compared against Standard (black), DenseWeight (orange), and RandomWeight (blue).
\textbf{Left:} Mean Absolute Error (MAE) at each autoregressive step, averaged over the full test set.
\textbf{Right:} Pearson correlation across the three state components at each step, averaged over the test set.
In both panels, the primary horizontal axis (bottom) gives the discrete step index; the secondary axis (top) indicates the corresponding physical timescale in Lyapunov Time (LT).
Solid lines are per-step means across the test set; shaded bands denote $\pm 1$ standard deviation across test-set trajectories at each step.
In the correlation panel, a horizontal dashed grey line marks the $0.5$ predictability threshold; coloured vertical dotted lines indicate the step at which each model's mean correlation first crosses this threshold.}
\label{fig:lorenz_ar_metrics}
\end{figure}
Fig.~\ref{fig:lorenz_ar_metrics} reports the closed-loop rollout performance. 
The ordering of methods established on KS is preserved across both metrics: DAW attains the lowest MAE and the slowest correlation decay throughout, followed by DenseWeight, RandomWeight, and Standard. 
Measured against the $0.5$ predictability threshold, DAW extends the horizon to $\approx 4.1$~LT, compared with $\approx 3.7$~LT for DenseWeight, $\approx 3.4$~LT for RandomWeight, and $\approx 2.3$~LT for Standard -- an increase of roughly $80\%$ over the uniform objective. 
As on KS, all methods eventually lose coherence under the intrinsic limits of chaotic predictability; the benefit of DAW is again a systematic delay of collapse rather than its elimination.

\section{Statistical significance of the win rate with bootstrapping}\label{si:bootstrap}

\subsection{Method}
\label{si:bootstrap-method}
 
The Overall Win Rate in Fig~\ref{fig:sample_error} is reported over the $21{,}146$ high-$d$ event windows extracted from the autoregressive rollouts. 
However, these windows are not mutually independent: windows drawn from the same rollout share an initial condition and are close in time, so treating each window as an independent sample may understate the sampling variability of the win rate. 
To obtain faithful confidence intervals, we therefore conduct a cluster (block) bootstrap at the level of the trajectories rather than the windows, i.e., we first compute the win rate over each trajectory, then aggregate over trajectories with resampled initial conditions~\cite{efron1993introduction}.
 
Let $\mathcal{T}$ index the $T = 13{,}111$ test trajectories that contribute at least one high-$d$ window; trajectories without an event carry no information about the win rate and are excluded. For each trajectory $t \in \mathcal{T}$ we precompute two sufficient statistics per method $m \in \{\mathrm{DAW}, \mathrm{DenseWeight}, \mathrm{RandomWeight}, \mathrm{Standard}\}$: the number of its windows won by $m$, $W_{t,m}$, and its total number of windows, $E_t$. A window is won by the method of
lowest Cumulative Error (CE); exact ties, which are negligible for continuous CE, are split fractionally across the tied methods so that each window contributes unit mass. 
The point estimate recovers the value reported in the main text,
$\widehat{r}_m = \sum_{t} W_{t,m} / \sum_{t} E_t$.
 
We then draw $B = 10{,}000$ bootstrap replicates. Each replicate samples $T$ trajectories from $\mathcal{T}$ with replacement and forms the win rate $r_m^{(b)} = \sum_{t \in \mathcal{S}_b} W_{t,m} \big/ \sum_{t \in \mathcal{S}_b} E_t$, where $\mathcal{S}_b$ is the resampled index set. Because both wins and events are aggregated to the trajectory before resampling, an entire trajectory -- with all of its within-cluster correlation intact -- is included or omitted as a unit, and each replicate reduces to a pair of sums. 
The $95\%$ confidence interval for each method is the $2.5$th--$97.5$th percentile of $\{r_m^{(b)}\}_{b=1}^{B}$.
 
The four win rates are compositional: they sum to one and are consequently negatively correlated, so a per-method interval alone does not establish that DAW ranks first. We therefore test each pairwise contrast directly. 
For every replicate we form the difference $\Delta_m^{(b)} = r_{\mathrm{DAW}}^{(b)} - r_m^{(b)}$ against each baseline $m$ and report its $95\%$ percentile interval. 
An interval lying entirely above zero indicates that DAW attains the higher win rate in essentially every resample, and thus a statistically significant advantage over that baseline.
The procedure is pure post-processing of the stored CE values and requires no model retraining.
 
\subsection{Results}
\label{si:bootstrap-results}
 
Table~\ref{tab:bootstrap} reports the bootstrapped win rates and the pairwise contrasts against DAW. The per-method intervals are narrow -- each spanning at most $1.4$ percentage points -- confirming that the point estimates are stable rather than artifacts of the particular test set. 
The intervals do not overlap across methods, so the ordering $\mathrm{DAW} > \mathrm{RandomWeight} > \mathrm{DenseWeight} > \mathrm{Standard}$ by event-level win rate is unambiguous at the $95\%$ level. 
In particular, the DAW win rate of $42.4\%$ ($95\%$ CI $[41.7, 43.1]$) sits well above the $25\%$ rate expected under chance among four competitors.
 
The pairwise contrasts place the main-text claim on a tested footing. DAW exceeds DenseWeight by $18.0$ points ($95\%$ CI $[16.8, 19.2]$), RandomWeight by $14.3$ points ($[13.0, 15.5]$), and Standard by $37.3$ points ($[36.5, 38.2]$); all three
intervals exclude zero by a wide margin. The advantage is therefore both statistically significant and of a magnitude far larger than the sampling uncertainty, even for the closest competitor. 
This upgrades the reported $42.4\%$ win rate from a point estimate to a tested statement, and corroborates the central finding that DAW suppresses per-event error accumulation at the system's most
demanding high-$d$ regimes more reliably than uniform training, purely statistical density weighting, or the shuffled-weight ablation.
 
\begin{table}[htbp]
\centering
\caption{\textbf{Cluster (block) bootstrap of the Overall Win Rate.} 
Bootstrap of the Overall Win Rate is calculated over the $21{,}146$
high-$d$ event windows, resampling the $T = 13{,}111$ event-bearing test
trajectories with replacement ($B = 10{,}000$). Win rates are given with $95\%$
percentile confidence intervals; the last column reports the paired difference
$\mathrm{DAW} - \text{method}$ with its $95\%$ interval. All contrasts against DAW
exclude zero.}
\label{tab:bootstrap}
\begin{tabular}{lcc}
\toprule
Method & Win rate (\%) [95\% CI] & $\Delta$ vs.\ DAW (\%) [95\% CI] \\
\midrule
DAW          & $42.4\ [41.7, 43.1]$ & --- \\
RandomWeight & $28.1\ [27.5, 28.8]$ & $+14.3\ [+13.0, +15.5]$ \\
DenseWeight  & $24.4\ [23.8, 25.0]$ & $+18.0\ [+16.8, +19.2]$ \\
Standard     & $\phantom{0}5.1\ [\phantom{0}4.8, \phantom{0}5.4]$ & $+37.3\ [+36.5, +38.2]$ \\
\bottomrule
\end{tabular}
\end{table}

\end{appendices}

\bibliographystyle{unsrt}  
\bibliography{references}

\end{document}